\documentclass[smallextended]{svjour3} 
\smartqed 
\usepackage[english]{babel}
\usepackage[utf8x]{inputenc}
\usepackage[T1]{fontenc}

\usepackage[a4paper,top=3cm,bottom=2cm,left=3cm,right=3cm,marginparwidth=1.75cm]{geometry}

\usepackage{amsmath,amsfonts,dsfont,amsthm,bm}
\usepackage{graphicx, algorithm, algorithmic}
\usepackage[colorinlistoftodos]{todonotes}
\usepackage[colorlinks=true, allcolors=blue]{hyperref}

\usepackage{float}
\usepackage{multirow}
\newcommand{\E}{\mathds{E}}
\let\P\undefined
\newcommand{\P}{\mathds{P}}
\newcommand{\R}{\mathbb{R}}
\let\t\undefined
\newcommand{\t}{\top}
\newcommand{\w}{\mathbf{w}}
\let\v\undefined
\newcommand{\v}{\mathbf{v}}
\newcommand{\x}{\mathbf{x}}
\newcommand{\Z}{\mathbf{Z}}
\newcommand{\z}{\mathbf{z}}
\newcommand{\g}{\mathbf{g}}
\newcommand{\q}{\mathbf{q}}
\newcommand{\I}{\mathbf{I}}
\newcommand{\1}{\mathbf{1}}
\newcommand{\0}{\mathbf{0}}
\newcommand{\Q}{\mathcal{Q}}

\newcommand{\p}{\mathrm{proj}}
\newcommand{\la}{\langle}
\newcommand{\ra}{\rangle}
\newcommand{\tv}{\tilde{\v}}
\newcommand{\tw}{\tilde{\w}}

\newtheorem{prop}{Proposition}
\newtheorem{thm}{Theorem}
\newtheorem{lem}{Lemma}
\newtheorem{cor}{Corollary}
\newtheorem{rem}{Remark}

\begin{document}

\title{Blended Coarse Gradient Descent for Full Quantization of Deep Neural Networks}
\author{Penghang Yin\thanks{P. Yin, S. Zhang and J. Lyu contributed equally.} \and Shuai Zhang \and Jiancheng Lyu \and Stanley Osher \and Yingyong Qi \and Jack Xin*
}
\titlerunning{Blended Coarse Gradient Descent} 

\institute{Penghang Yin \and Stanley Osher \at
              Department of Mathematics, University of California at Los Angeles, Los Angeles, CA 90095 \\
              \email{yph@ucla.edu, sjo@math.ucla.edu}           
           \and
           Shuai Zhang \and Yingyong Qi \at
              Qualcomm AI Research, San Diego, CA 92121 \\
              \email{shuazhan@qti.qualcomm.com, yingyong@qti.qualcomm.com}
          \and
          Jiancheng Lyu \and Jack Xin \at
          Department of Mathematics, University of California at Irvine, Irvine, CA 92697 \\
          \email{jianchel@uci.edu; jxin@math.uci.edu, *corresponding author, (949)-331-6314.}
}

\date{Received: date / Accepted: date}

\maketitle

\begin{abstract}
Quantized deep neural networks (QDNNs) are attractive due to their much lower memory storage and faster inference speed than their regular full precision counterparts. To maintain the same performance level especially at low bit-widths, QDNNs must be retrained. Their training involves piecewise constant activation functions and discrete weights, hence mathematical challenges arise. We introduce the notion of coarse gradient and propose the blended coarse gradient descent (BCGD) algorithm, for training fully quantized neural networks. Coarse gradient is generally not a gradient of any function but an artificial ascent direction. The weight update of BCGD goes by coarse gradient correction of a weighted average of the full precision weights and their quantization (the so-called blending), which yields sufficient descent in the objective value and thus accelerates the training. Our experiments demonstrate that this simple blending technique is very effective for quantization at extremely low bit-width such as binarization. In full quantization of ResNet-18 for ImageNet classification task, BCGD gives 64.36\% top-1 accuracy with binary weights across all layers and 4-bit adaptive activation. If the weights in the first and last layers are kept in full precision, this number increases to 65.46\%. As theoretical justification, we show convergence analysis of coarse gradient descent for a two-linear-layer neural network model with Gaussian input data, and prove that the expected coarse gradient correlates positively with the underlying true gradient. 


\end{abstract}
\thispagestyle{empty}

\newpage 
\section{Introduction}
\setcounter{page}{1}
Deep neural networks (DNNs) have seen enormous success in image and speech classification, natural language processing, health sciences  among other big data driven 
applications in recent years. 
However, DNNs typically require hundreds of megabytes of memory storage for
the trainable full-precision 
floating-point parameters, and billions of FLOPs
(floating point operations per second) to make a single inference. This makes the deployment of DNNs on mobile 
devices a challenge. Some considerable recent efforts have been devoted to the training of low precision  (quantized) models for substantial memory savings and computation/power efficiency, while nearly maintaining the performance 
of full-precision networks. Most works to date 
are concerned with weight quantization (WQ) 
\cite{bc_15,twn_16,xnornet_16,lbwn_16,carreira1,BR_18}. In \cite{he2018relu}, He et al. theoretically justified for the applicability of WQ models by investigating their expressive power. Some also studied activation function quantization (AQ) \cite{bnn_16,xnornet_16,Hubara2017QuantizedNN,halfwave_17,entropy_17,dorefa_16}, which utilize an external process outside of the network training. This is different from WQ at 4 bit or under, which must be achieved through network training. Learning activation function $\sigma$ as a parametrized family ($\sigma=\sigma(x,\alpha)$) and part of network training has been studied in \cite{paramrelu} for parametric rectified linear unit, and was recently extended to uniform AQ in \cite{pact}. In uniform AQ, $\sigma(x,\alpha)$ is a step (or piecewise constant) function in $x$, and the parameter $\alpha$ determines the height and length of the steps. In terms of the partial derivative of $\sigma(x,\alpha)$ in $\alpha$, a two-valued proxy derivative of the parametric activation function (PACT) was proposed \cite{pact}, although we will present an almost everywhere (a.e.) exact one in this paper.

\medskip

The mathematical difficulty in training activation quantized networks is that the loss function becomes a piecewise constant function with sampled stochastic gradient a.e. zero, which is undesirable for back-propagation. A simple and effective way around this problem is to use a (generalized) straight-through (ST) estimator or derivative of a related (sub)differentiable function \cite{hinton2012neural,bengio2013estimating,bnn_16,Hubara2017QuantizedNN} such as clipped rectified linear unit (clipped ReLU) \cite{halfwave_17}. The idea of ST estimator dates back to the perceptron algorithm \cite{rosenblatt1957perceptron,rosenblatt1962principles} proposed in 1950s for learning single-layer perceptrons with binary output. For multi-layer networks with hard threshold activation (a.k.a. binary neuron), Hinton \cite{hinton2012neural} proposed to use the derivative of identity function as a proxy in back-propagation or chain rule, similar to the perceptron algorithm. The proxy derivative used in backward pass only was referred as straight-through estimator in \cite{bengio2013estimating}, and several variants of ST estimator \cite{bnn_16,Hubara2017QuantizedNN,halfwave_17} have been proposed for handling quantized activation functions since then. A similar situation, where the derivative of certain layer composited in the loss function is unavailable for back-propagation, has also been brought up by \cite{wang2018deep} recently while improving accuracies of DNNs by replacing the softmax classifier layer with an implicit weighted nonlocal Laplacian layer. For the training of the latter, the derivative of a pre-trained fully-connected layer was used as a surrogate \cite{wang2018deep}. 

\medskip

On the theoretical side, while the convergence of the single-layer perception algorithm has been extensively studied \cite{widrow199030,freund1999large}, there is almost no theoretical understanding of the unusual `gradient' output from the modified chain rule based on ST estimator. Since this unusual `gradient' is certainly not the gradient of the objective function, then a question naturally arises: \emph{how does it correlate to the objective function?} One of the contributions in this paper is to answer this question. Our main contributions are threefold:

\begin{enumerate}
\item Firstly, we introduce the notion of coarse derivative and cast the early ST estimators or proxy partial derivatives of $\sigma(x,\alpha)$ in $\alpha$ including the two-valued PACT of \cite{pact} as examples. The coarse derivative is non-unique. We propose a three-valued coarse partial derivative of the quantized activation function $\sigma(x,\alpha)$ in $\alpha$ that can outperform the two-valued one \cite{pact} in network training. We find that unlike the partial derivative $\frac{\partial \sigma}{\partial x}(x,\alpha)$ which vanishes, the a.e. partial derivative of $\sigma(x,\alpha)$ in $\alpha$ is actually multi-valued (piecewise constant). Surprisingly, this a.e. accurate derivative is empirically less useful than the coarse ones in fully quantized network training. \\

\item Secondly, we propose a novel accelerated training algorithm for fully quantized networks, termed blended coarse gradient descent method (BCGD). 
Instead of correcting the current full precision weights with coarse gradient at their quantized values like in the popular BinaryConnect scheme \cite{bc_15,bnn_16,xnornet_16,twn_16,dorefa_16,halfwave_17,Goldstein_17,BR_18}, the BCGD weight update goes by coarse gradient correction of a suitable average of the full precision weights and their quantization. We shall show that BCGD satisfies the sufficient descent property for objectives with Lipschitz gradients, while BinaryConnect does not unless an approximate orthogonality condition holds for the iterates \cite{BR_18}. \\

\item Our third contribution is the mathematical analysis of coarse gradient descent for a two-layer network with binarized ReLU activation function and i.i.d. unit Gaussian data. We provide an explicit form of coarse gradient based on proxy derivative of regular ReLU, and show that when there are infinite training data, the negative expected coarse gradient gives a descent direction for minimizing the expected training loss. Moreover, we prove that a normalized coarse gradient descent algorithm only converges to either a global minimum or a potential spurious local minimum. This answers the question.

\end{enumerate}

\medskip

The rest of the paper is organized as follows. In section 2, we discuss the concept of coarse derivative and give examples for quantized activation functions. In section 3, we present the joint weight and activation quantization problem, 
and BCGD algorithm satisfying the sufficient descent property. For readers' convenience, we also review formulas on 1-bit, 2-bit and 4-bit weight quantization used later in our numerical experiments. In section 4, we give details of fully quantized network training, including the disparate learning rates on weight and $\alpha $. We illustrate the enhanced validation accuracies of BCGD over BinaryConnect, and 3-valued coarse $\alpha$ partial derivative of $\sigma$ over 2-valued and a.e. $\alpha$ partial derivative in case of 4-bit activation, and (1,2,4)-bit weights on CIFAR-10 image datasets. We show top-1 and top-5 validation accuracies of ResNet-18 with all convolutional layers quantized at 1-bit weight/4-bit activation (1W4A), 4-bit weight/4-bit activation (4W4A), 
and 4-bit weight/8-bit activation (4W8A), using 3-valued and 2-valued $\alpha$ partial derivatives. The 3-valued $\alpha$ partial derivative out-performs the two-valued with larger margin in the low bit regime. 
The accuracies degrade gracefully from 4W8A to 1W4A while all the convolutional layers are quantized.
The 4W8A accuracies with either the 3-valued 
or the 2-valued $\alpha$ partial derivatives are within 1\% of those of the full precision network. If the first and last convolutional layers are in full precision, our top-1 (top-5) accuracy of ResNet-18 at 1W4A with 3-valued coarse $\alpha$-derivative is 4.7 \% (3\%) higher than that of HWGQ \cite{halfwave_17} on ImageNet dataset. This is in part due to the value of parameter $\alpha$ being learned without any statistical assumption.

\bigskip

\noindent{\bf Notations.}
$\|\cdot\|$ denotes the Euclidean norm of a vector or the spectral norm of a matrix; $\|\cdot\|_\infty$ denotes the $\ell_\infty$-norm. $\0\in\R^n$ represents the vector of zeros, whereas $\1\in\R^n$ the vector of all ones. We denote vectors by bold small letters and matrices by bold capital ones. For any $\w, \, \z\in\R^n$, $\w^\t\z = \la \w, \z \ra = \sum_{i} w_i z_i$ is their inner product. $\w\odot\z$ denotes the Hadamard product whose $i$-th entry is given by $(\w\odot\z)_i = w_iz_i$. 

\section{Activation Quantization}\label{sec:activ}
In a network with quantized activation, given a training sample of input $\Z$ and label $u$, the associated sample loss is a composite function of the form:
\begin{equation}\label{sLoss}
\ell(\w, \bm{\alpha} ; \{ \Z, u \}) := \ell(\w_l*\sigma(\w_{l-1}*\cdots \w_2*\sigma(\w_1*\Z, \alpha_1) \cdots, \alpha_{l-1} ); \, u),
\end{equation}
where $\w_j$ contains the weights in the $j$-th linear (fully-connected or convolutional) layer, `$*$' denotes either matrix-vector product or convolution operation; reshaping is necessary to avoid mismatch in dimensions. The $j$-th quantized ReLU $\sigma(\x_j,\alpha_j)$ acts element-wise on the vector/tensor $\x_j$ output from the previous linear layer, which is parameterized by a trainable scalar $\alpha_j>0$ known as the resolution. For practical hardware-level implementation, we are most interested in uniform quantization:
\begin{equation}\label{eq:qrelu}
\sigma\left(x,\alpha\right) = \begin{cases}
0, \quad & \mathrm{if}\quad x \leq 0,\\
k\alpha, \quad & \mathrm{if}\quad \left(k-1\right)\alpha < x \leq k\alpha, \; k = 1, 2, \dots, 2^{b_a}-1,\\
\left(2^{b_a}-1\right)\alpha,\quad & \mathrm{if}\quad x > \left(2^{b_a}-1\right)\alpha,
\end{cases}
\end{equation}
where $x$ is the scalar input, $\alpha>0$ the resolution, $b_a \in \mathbb{Z}_+$ the bit-width of activation and $k$ the quantization level. For example, in 4-bit activation quantization (4A), we have $b_a = 4$ and $2^{b_a} = 16$ quantization levels including the zero. 

\medskip

Given $N$ training samples, we train the network with quantized ReLU by solving the following empirical risk minimization
\begin{equation}\label{eq:training}
\min_{\w,\boldsymbol\alpha}\, f(\w,\boldsymbol\alpha) :=  \frac{1}{N}\sum_{i=1}^N \; \ell(\w, \bm{\alpha} ; \{ \Z^{(i)},u^{(i)} \})
\end{equation}
In gradient-based training framework, one needs to evaluate the gradient of the sample loss (\ref{sLoss}) using the so-called back-propagation (a.k.a. chain rule), which involves the computation of partial derivatives $\frac{\partial \sigma}{\partial x}$ and $\frac{\partial \sigma}{\partial \alpha}$. Apparently, the partial derivative of $\sigma\left(x,\alpha\right)$ in $x$ is almost everywhere (a.e.) zero. After composition, this results in a.e. zero gradient of $\ell$ with respect to (w.r.t.) $\{\w_j\}_{j=1}^{l-1}$ and $\{\alpha_j\}_{j=1}^{l-2}$ in (\ref{sLoss}), causing their updates to become stagnant. To see this, we abstract the partial gradients $\frac{\partial \ell}{\partial \w_{l-1}}$ and $\frac{\partial \ell}{\partial \alpha_{l-2}}$, for instances, through the chain rule as follows:
$$
\frac{\partial \ell}{\partial \w_{l-1}}(\w,\bm{\alpha}; \{\Z,u\}) = \sigma(\x_{l-2},\alpha_{l-2}) \circ \frac{\partial\sigma}{\partial x}(\x_{l-1},\alpha_{l-1})\circ \w_l^\t \circ \nabla\ell(\x_l; u) 
$$
and
$$
\frac{\partial \ell}{\partial \alpha_{l-2}}(\w,\bm{\alpha}; \{\Z,u\}) = \frac{\partial \sigma}{\partial \alpha}(\x_{l-2},\alpha_{l-2}) \circ \w_{l-1}^\t \circ \frac{\partial\sigma}{\partial x}(\x_{l-1},\alpha_{l-1})\circ \w_l^\t \circ \nabla\ell(\x_l; u), 
$$
where we recursively define $\x_1 = \w_1*\Z$, and $\x_j = \w_j*\sigma(\x_{j-1}, \alpha_{j-1})$ for $j\geq2$ as the output from the $j$-th linear layer, and `$\circ$' denotes some sort of proper composition in the chain rule. It is clear that the two partial gradients are zeros a.e. because of the term $\frac{\partial\sigma}{\partial x}(\x_{l-1},\alpha_{l-1})$. In fact, the automatic differentiation embedded in deep learning platforms such as PyTorch \cite{pytorch} would produce precisely zero gradients.

\medskip

To get around this, we use a proxy derivative or so-called ST estimator for back-propagation. By overloading the notation `$\approx$', we denote the proxy derivative by
\begin{align*}
\frac{\partial \sigma}{\partial x}\left(x,\alpha\right) \approx
\begin{cases}
0, \quad & \mathrm{if}\quad x \leq 0,\\
1, \quad & \mathrm{if}\quad 0 < x \leq \left(2^{b_a}-1\right)\alpha,\\
0,\quad & \mathrm{if}\quad x > \left(2^{b_a}-1\right)\alpha.
\end{cases}
\end{align*}
The proxy partial derivative has a non-zero value in the middle to reflect the overall variation of $\sigma$, which can be viewed as the derivative of the large scale (step-back) view of $\sigma$ in $x$, or the derivative of the clipped ReLU \cite{halfwave_17}:
\begin{equation}\label{eq:crelu}
\tilde{\sigma}(x,\alpha) = 
\begin{cases}
0, \quad & \mathrm{if}\quad x \leq 0,\\
x, \quad & \mathrm{if}\quad 0 < x \leq (2^{b_a}-1)\alpha, \\
\left(2^{b_a}-1\right)\alpha,\quad & \mathrm{if}\quad x > \left(2^{b_a}-1\right)\alpha.
\end{cases}
\end{equation}

\begin{figure}[ht]
\centering
\begin{tabular}{cc}
\includegraphics[width=0.48\textwidth]{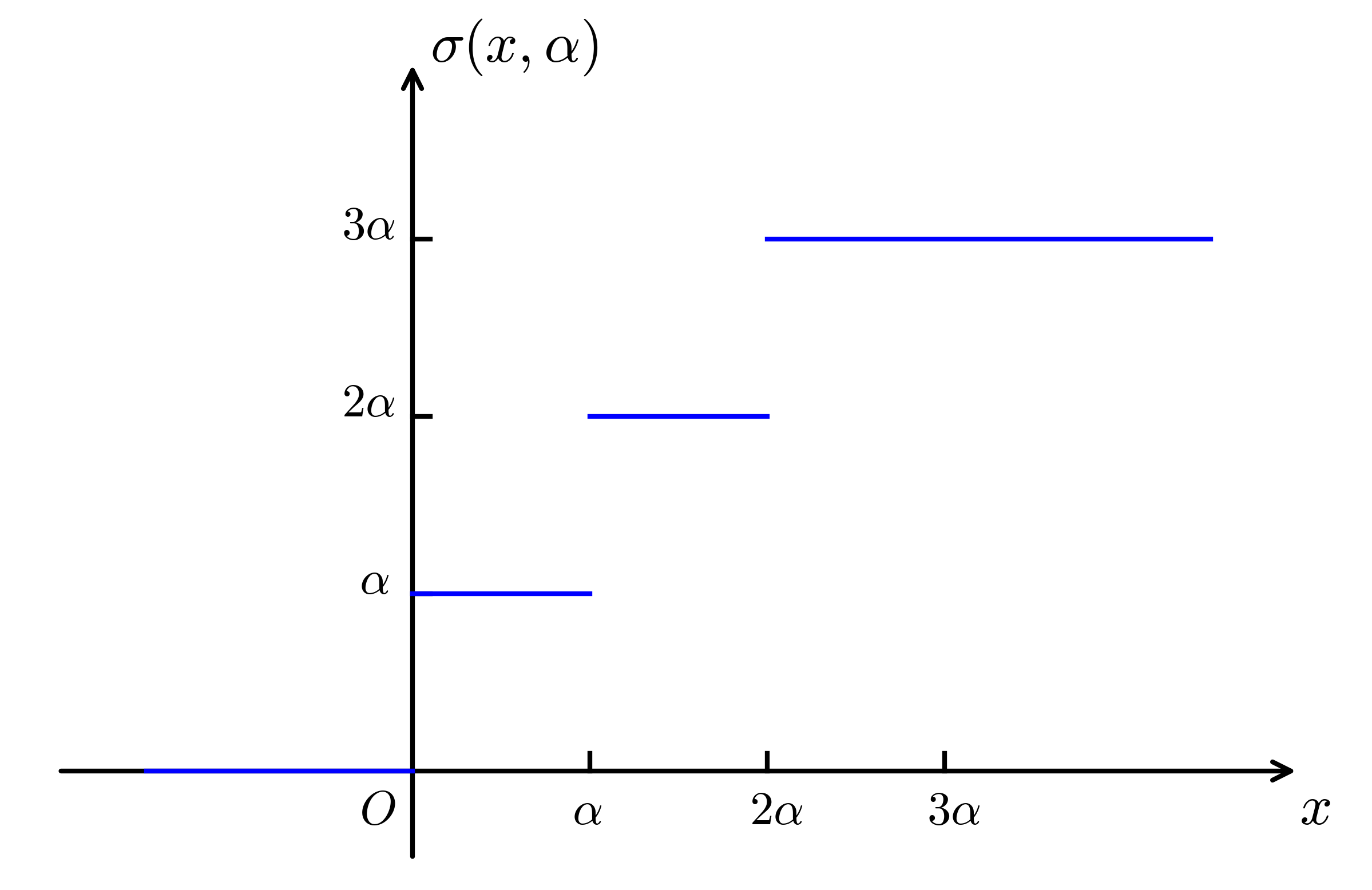}
\includegraphics[width=0.48\textwidth]{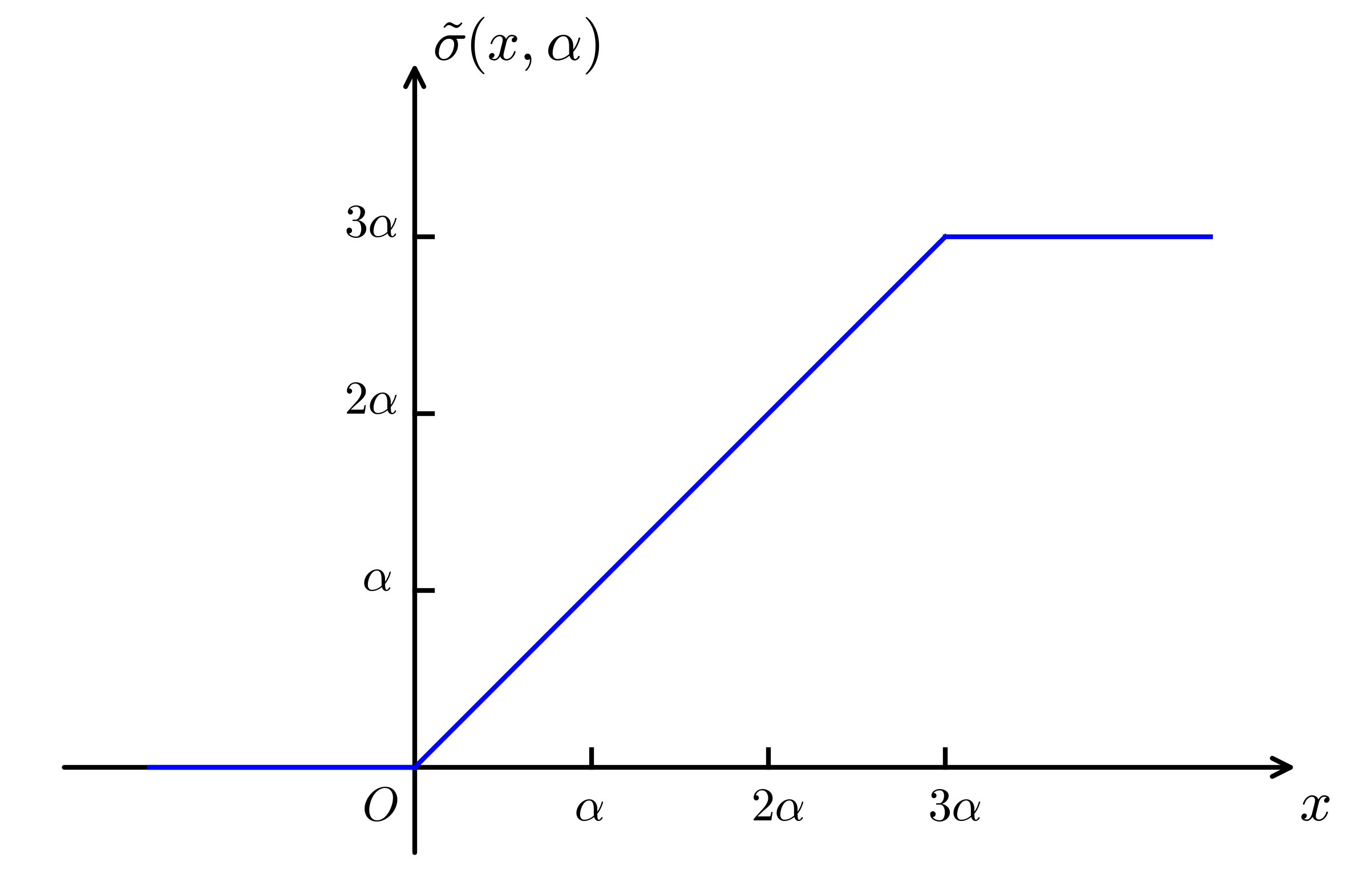}
\end{tabular}
\caption{Left: plot of 2-bit quantized ReLU $\sigma(x,\alpha)$ in $x$. Right: plot of the associated clipped ReLU $\tilde{\sigma}(x,\alpha)$ in $x$.}\label{fig:relu}
\end{figure}


On the other hand, we find the a.e. partial derivative of $\sigma(x,\alpha)$ w.r.t. $\alpha$ to be
\begin{align*}
\dfrac{\partial \sigma}{\partial \alpha}(x,\alpha) = \begin{cases}
0, \quad & \mathrm{if}\quad x \leq 0,\\
k, \quad & \mathrm{if}\quad \left(k-1\right)\alpha < x \leq k\alpha,\;\; k=1,2,\cdots, 2^{b_a} -1;\\
2^{b_a}-1,\quad & \mathrm{if}\quad x > \left(2^{b_a}-1\right)\alpha.
\end{cases}
\end{align*}

\noindent Surprisingly, this a.e. derivative is not the best in terms of accuracy or computational cost in training, as will be reported in section \ref{sec:exper}. We propose an empirical three-valued proxy partial derivative in $\alpha$ as follows
\begin{align*}
\dfrac{\partial \sigma}{\partial \alpha}(x,\alpha) \approx \begin{cases}
0, \quad & \mathrm{if}\quad x \leq 0,\\
2^{(b_a-1)}, \quad & \mathrm{if}\quad 0 < x \leq \left(2^{b_a}-1\right)\alpha,\\
2^{b_a}-1,\quad & \mathrm{if}\quad x > \left(2^{b_a}-1\right)\alpha.
\end{cases}
\end{align*}
The middle value $2^{b_a-1}$ is the arithmetic mean of the intermediate $k$ values of the a.e. partial derivative above. Similarly, a more coarse two-valued proxy, same as PACT \cite{pact} which was derived differently, follows by zeroing out all 
the nonzero values except their maximum:
\begin{align*}
\dfrac{\partial \sigma}{\partial \alpha}(x,\alpha) \approx \begin{cases}
0, \quad & \mathrm{if}\quad x \leq \left(2^{b_a}-1\right)\alpha,\\
2^{b_a}-1,\quad & \mathrm{if}\quad x > \left(2^{b_a}-1\right)\alpha.
\end{cases}
\end{align*}
This turns out to be exactly the partial derivative $\dfrac{\partial \tilde{\sigma}}{\partial \alpha}(x,\alpha)$ of the clipped ReLU defined in (\ref{eq:crelu}).

\medskip

We shall refer to the resultant composite `gradient' of $f$ through the modified chain rule and averaging as coarse gradient. While given the name `gradient', we believe it is generally not the gradient of any smooth function.  It, nevertheless, somehow exploits the essential information of the piecewise constant function $f$, and its negation provides a descent direction for the minimization. In section \ref{sec:analy}, we will validate this claim by examining a two-layer network with i.i.d. Gaussian data. We find that when there are infinite number of training samples, the overall training loss $f$ (i.e., population loss) becomes pleasantly differentiable whose gradient is non-trivial and processes certain Lipschitz continuity. More importantly, we shall show an example of expected coarse gradient that provably forms an acute angle with the underlying true gradient of $f$ and only vanishes at the possible local minimizers of the original problem.

\medskip

During the training process, the vector $\bm{\alpha}$ (one component per activation layer) should be prevented from being either too small or too large. Due to the sensitivity of $\bm{\alpha}$, we propose a two-scale training and set the learning rate of $\bm{\alpha}$ to be the learning rate of weight $\w$ multiplied by a rate factor far less than 1, which may be varied depending on network architectures. That rate factor effectively helps quantized network converge steadily and prevents $\bm{\alpha}$ from vanishing.

\section{Full Quantization}
Imposing the quantized weights amounts to adding a discrete set-constraint $\w\in\Q$ to the optimization problem (\ref{eq:training}). Suppose $M$ is the total number of weights in the network. For commonly used $b_w$-bit layer-wise quantization, $\Q\subset\R^M$ takes the form of $\Q_1 \times \Q_2 \cdots\times \Q_l$, meaning that the weight tensor in the $j$-th linear layer is constrained in the form $\w_j = \delta_j \q_j \in\Q_j$ for some adjustable scaling factor $\delta_j>0$ shared by weights in the same layer. Each component of $\q_j$ is drawn from the quantization set given by $\{\pm 1\}$ for $b_w = 1$ (binarization) and $\{0,\pm 1,\cdots, \pm(2^{b_w-1} -1)\}$ for $b_w \geq 2$. This assumption on $\Q$ generalizes those of the 1-bit BWN \cite{xnornet_16} and the 2-bit TWN \cite{twn_16}.
As such, the layer-wise weight and activation quantization problem here can be stated abstractly as follows
\begin{equation}
\min_{\w,\boldsymbol\alpha}\, f(\w,\boldsymbol\alpha) \;\; \mbox{subject to} \;\; \w \in \Q=\Q_1 \times \Q_2 \cdots\times \Q_l, \label{e1}
\end{equation}
where the training loss $f(\w, \bm{\alpha})$ is defined in (\ref{eq:training}). Different from activation quantization, one bit is taken to represent the signs. For ease of presentation, we only consider the network-wise weight quantization throughout this section, i.e., weights across all the layers share the same (trainable) floating scaling factor $\delta>0$, or simply, $\Q = \R_+\times\{\pm 1\}^M$ for $b_w =1$ and $\Q = \R_+\times\left\{0,\pm 1, \dots, \pm(2^{b_w-1}-1)\right\}^M$ for $b_w\geq2$.

\subsection{Weight Quantization}
Given a float weight vector $\w_f$, the quantization of $\w_f$ is basically the following optimization problem for computing the projection of $\w_f$ onto set $\Q$
\begin{equation}\label{eq:proj}
\p_\Q(\w_f) := \arg\min_{\w\in\Q} \; \|\w- \w_f\|^2 .
\end{equation}
Note that $\Q$ is a non-convex set, so the solution of (\ref{eq:proj}) may not be unique. When $b_w=1$, we have the binarization problem 
\begin{equation}\label{eq:binary}
\min_{\delta, \q} \; \|\delta \, \q -\w_f\|^2 \quad \mbox{subject to} \quad \delta>0, \;  \q\in \left\{\pm1\right\}^M.
\end{equation}
For $b_w\geq 2$, the projection/quantization problem (\ref{eq:proj}) can be reformulated as 
\begin{equation}\label{eq:quant}
\min_{\delta, \q} \; \|\delta \, \q -\w_f\|^2 \quad \mbox{subject to} \quad \delta>0, \;  \q\in \left\{0,\pm1,\cdots, \pm (2^{b_w-1} -1)\right\}^M.
\end{equation}
It has been shown that the closed form (exact) solution of (\ref{eq:binary}) can be computed at $O(M)$ complexity for (1-bit) binarization \cite{xnornet_16} and at $O(M\log(M))$ complexity for (2-bit) ternarization \cite{lbwn_16}. An empirical ternarizer of $O(M)$ complexity has also been proposed \cite{twn_16}. At wider bit-width $b_w\geq 3$, accurately solving (\ref{eq:quant}) becomes computationally intractable due to the combinatorial nature of the problem \cite{lbwn_16}.  

\medskip

The problem (\ref{eq:quant}) is basically a constrained $K$-means clustering problem of 1-D points \cite{BR_18} with the centroids being $\delta$-spaced. It in principle can be solved by a variant of the classical Lloyd's algorithm \cite{lloyd} via an alternating minimization procedure. It iterates between the assignment step ($\q$-update) and centroid step ($\delta$-update). In the $i$-th iteration, fixing the scaling factor $\delta^{i-1}$, each entry of $\q^{i}$ is chosen from the quantization set, so that $\delta^{i-1} \q^i$ is as close as possible to $\w_f$. In the $\delta$-update, the following quadratic problem 
$$
\min_{\delta\in\R} \; \|\, \delta \, \q^{i} - \w_f\|^2  
$$
is solved by $\delta^{i} = \frac{(\q^{i})^\t\w_f}{\|\q^{i}\|^2}$. Since quantization (\ref{eq:proj}) is required in every iteration, to make this procedure practical, we just perform a single iteration of Lloyd's algorithm by empirically initializing $\delta$ to be $\frac{2}{2^{b_w} - 1} \|\w_f\|_\infty$, which is derived by setting
$$
\frac{\delta}{2}\left((2^{b_w-1}-1) + 2^{b_w-1}\right) = \|\w_f\|_\infty.
$$
This makes the large components in $\w_f$ well clustered.

\medskip

First introduced in \cite{bc_15} by Courbariaux et al., the BinaryConnect (BC) scheme has drawn much attention in training DNNs with quantized weight and regular ReLU. It can be summarized as
$$
\w_f^{t+1} = \w_f^t - \eta \nabla f(\w^t), \; \w^{t+1} = \p_\Q(\w_f^{t+1}),
$$
where $\{\w^t\}$ denotes the sequence of the desired quantized weights, and $\{\w^t_f\}$ is an auxiliary sequence of floating weights. BC can be readily extended to full quantization regime by including the update of $\bm{\alpha}^t$ and replacing the true gradient $\nabla f(\w^t)$ with the coarse gradients from section \ref{sec:activ}. With a subtle change to the standard projected gradient descent algorithm (PGD) \cite{combettes2015stochastic}, namely
$$
\w_f^{t+1} = \w^t - \eta \nabla f(\w^t), \; \w^{t+1} = \p_\Q(\w_f^{t+1}),
$$
BC significantly outperforms PGD and effectively bypasses spurious the local minima in $\Q$ \cite{Goldstein_17}. An intuitive explanation is that the constraint set $\Q$ is basically a finite union of isolated one-dimensional subspaces (i.e., lines that pass through the origin) \cite{BR_18}. Since $\w_f^t$ is obtained near the projected point $\w^t$, the sequence $\{\w_f^t\}$ generated by PGD can get stuck in some line subspace easily when updated with a small learning rate $\eta$; see Figure \ref{fig:alg} for graphical illustrations.

\begin{figure}[ht]
\centering
\begin{tabular}{cc}
\includegraphics[width=0.48\textwidth]{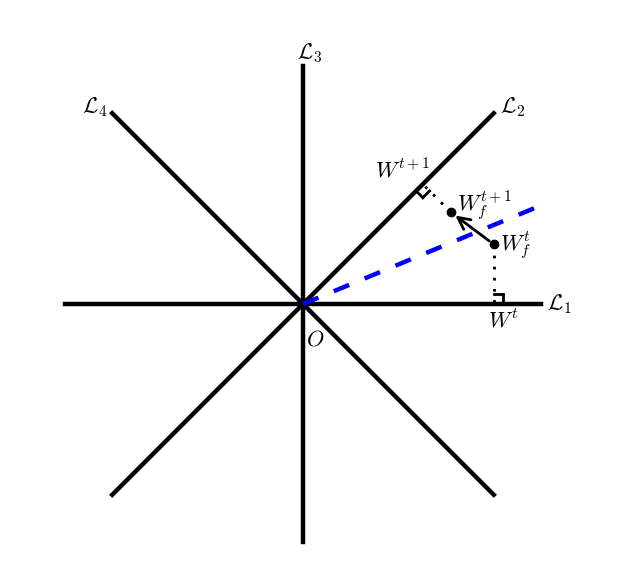}
\includegraphics[width=0.48\textwidth]{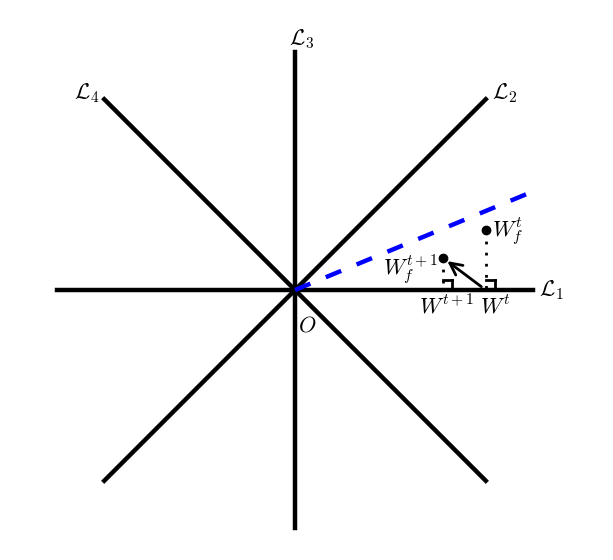}
\end{tabular}
\caption{The ternarization of two weights. The one-dimensional subspaces $\mathcal{L}_i$'s constitute the constraint set $\Q = \R_{+}\times\{0, \pm 1\}^2$. When updated with small learning rate, BC keeps searching among the subspaces (left), whereas PGD can get stagnated in $\mathcal{L}_1$ (right).
}\label{fig:alg}. 
\end{figure}

\subsection{Blended Gradient Descent and Sufficient Descent Property}
Despite the superiority of BC over PGD, we point out a drawback in regard to its convergence. While Yin et al. provided the convergence proof of BC scheme in the recent papers \cite{BR_18}, their analysis hinges on an approximate orthogonality condition which may not hold in practice; see Lemma 4.4 and Theorem 4.10 of \cite{BR_18}. Suppose $f$ has $L$-Lipschitz gradient\footnote{This assumption is valid for the population loss function; we refer readers to Lemma \ref{lem:lipschitz} in section \ref{sec:analy}.}. In light of the convergence proof in Theorem 4.10 of \cite{BR_18}, we have
\begin{equation}\label{eq:desc}
f(\w^{t+1})-f(\w^t)\leq - \frac{1}{2}\left(\frac{1}{\eta} ( \|\w^{t+1}-\w_f^t\|^2 - \|\w^t -\w_f^t\|^2 ) - L\|\w^{t+1}-\w^t\|^2\right).
\end{equation}
For the objective sequence $\{f(\w^t)\}$ to be monotonically decreasing and $\{\w^k\}$ converging to a critical point, it is crucial to have the sufficient descent property \cite{gilbert1992global} hold for sufficiently small learning rate $\eta>0$: 
\begin{equation}
f(\w^{t+1}) - f(\w^t) \leq - c \, \|\w^{t+1} -\w^{t}\|^2, \label{descent}
\end{equation}
with some positive constant $c>0$. 

\medskip

Since $\w^{t} = \arg\min_{\w\in\Q} \|\w - \w_f^t\|^2$ and $\w^{t+1}\in\Q$, it holds in (\ref{eq:desc}) that 
$$
\frac{1}{\eta}(\|\w^{t+1}-\w_f^t\|^2 - \|\w^t -\w_f^t\|^2)\geq 0.
$$
Due to non-convexity of the set $\Q$, the above term can be as small as zero even when $\w^t$ and $\w^{t+1}$ are distinct. So it is not guaranteed to dominate the right hand side of (\ref{eq:desc}). Consequently given (\ref{eq:desc}), the inequality (\ref{descent}) does not necessarily hold. Without sufficient descent, even if $\{f(\w^t)\}$ converges, the iterates $\{\w^t\}$ may not converge well to a critical point. To fix this issue, we blend the ideas of  PGD and BC, and propose the following blended gradient descent (BGD)
\begin{equation}\label{bcgd}
\w_f^{t+1} = (1-\rho)\w_f^t + \rho \w^t - \eta \nabla f(\w^t), \; \w^{t+1} = \p_\Q(\w_f^{t+1})
\end{equation}
for some blending parameter $\rho\ll 1$. In contrast, the blended gradient descent satisfies (\ref{descent}) for small enough $\eta$.
\begin{prop}\label{prop:bgd}
For $\rho\in(0,1)$, the BGD (\ref{bcgd}) satisfies
$$
f(\w^{t+1})-f(\w^t)\leq - \frac{1}{2}\left(\frac{1-\rho}{\eta}(\|\w^{t+1}-\w_f^t\|^2 - \|\w^t -\w_f^t\|^2) + \left(\frac{\rho}{\eta}-L \right)\|\w^{t+1}-\w^t\|^2\right).
$$
\end{prop}

Choosing the learning rate $\eta$ small enough so that $\rho/\eta \geq L + c$. Then inequality (\ref{descent}) follows from the above proposition, which will guarantee the convergence of (\ref{bcgd}) to a critical point by using similar arguments as in the proofs from \cite{BR_18}. 
\begin{cor}
The blended gradient descent iteration (\ref{bcgd}) satisfies the sufficient descent property (\ref{descent}).
\end{cor}

\section{Experiments}\label{sec:exper}

We tested BCGD, as summarized in Algorithm \ref{alg}, on the CIFAR-10 \cite{cifar_09} and ImageNet \cite{imagenet_09,imagnet_12} color image datasets. We coded up the BCGD in PyTorch platform \cite{pytorch}. In all experiments, we fix the blending factor in (\ref{bcgd}) to be $\rho = 10^{-5}$. All runs with quantization are warm started with a float pre-trained model, and the resolutions $\bm{\alpha}$ are initialized by $\frac{1}{2^{b_a} - 1}$ of the maximal values in the corresponding feature maps generated by a random mini-batch. The learning rate for weight $\w$ starts from $0.01$. Rate factor for the learning rate of $\bm{\alpha}$ is $0.01$, i.e., the learning rate for $\bm{\alpha}$ starts from $10^{-4}$. The decay factor for the learning rates is $0.1$. The weights $\w$ and resolutions $\bm{\alpha}$ are updated jointly. In addition, we used momentum and batch normalization \cite{bn_15} to promote training efficiency. We mainly compare the performances of the proposed BCGD and the state-of-the-art BC (adapted for full quantization) on \emph{layer-wise} quantization. The experiments were carried out on machines with 4 Nvidia GeForce GTX 1080 Ti GPUs.

\begin{algorithm}
\caption{One iteration of BCGD for full quantization}\label{alg}
\textbf{Input}:  mini-batch loss function $f_t(\w,\bm{\alpha})$, blending parameter $\rho=10^{-5}$, learning rate $\eta_\w^t$ for the weights $\w$, learning rate $\eta_{\bm{\alpha}}^t$ for the resolutions $\bm{\alpha}$ of AQ (one component per activation layer). \\
\textbf{Do}:
\begin{algorithmic}
    \STATE Evaluate the mini-batch coarse gradient $(\tilde{\nabla}_\w f_t, \tilde{\nabla}_{\bm{\alpha}} f_t)$ at $(\w^t, \bm{\alpha}^t)$ according to section 2.
    \STATE $\w_f^{t+1} = (1-\rho)\w_f^t + \rho\w^t - \eta^t_\w \tilde{\nabla}_\w f_t(\w^t,\bm{\alpha}^t)$ \quad $//$ blended gradient update for weights
    \STATE $\bm{\alpha}^{t+1} = \bm{\alpha}^t - \eta^t_{\bm{\alpha}}  \tilde{\nabla}_{\bm{\alpha}} f_t(\w^t,\bm{\alpha}^t)$ \quad $//$ $\eta_{\bm{\alpha}}^t = 0.01\cdot \eta_\w^t$
    \STATE $\w^{t+1} = \p_\Q(\w_f^{t+1})$  \quad $//$ quantize the weights as per section 3.1
\end{algorithmic}
\end{algorithm}


The CIFAR-10 dataset consists of 60,000 $32\times32$ color images of 10 classes, with 6,000 images per class. There dataset is split into 50,000 training images and 10,000 test images. In the experiments, we used the testing images for validation. The mini-batch size was set to be $128$ and the models were trained for $200$ epochs with learning rate decaying at epoch 80 and 140. In addition, we used weight decay of $10^{-4}$ and momentum of $0.95$. The a.e derivative, 3-valued and 2-valued coarse derivatives of $\bm{\alpha}$ are compared on the VGG-11 \cite{vgg_14} and ResNet-20 \cite{resnet_15} architectures, and the results are listed in Tables \ref{tab:1}, \ref{tab:2} and \ref{tab:3}, respectively. It can be seen that the 3-valued coarse $\bm{\alpha}$ derivative gives the best overall performance in terms of accuracy. Figure \ref{fig:1} shows that in weight binarization, BCGD converges faster and better than BC. 

\begin{table}[ht]
\centering
\begin{tabular}{|c|c|c|c|c|c|c|c|}
  \hline			
 Network & Float & 32W4A & 1W4A & 2W4A & 4W4A  \\
  \hline
  VGG-11 + BC & \multirow{2}{*}{92.13} & \multirow{2}{*}{91.74} & 88.12 & 89.78 & {\bf  91.51} \\
  \cline{1-1}\cline{4-6}
  VGG-11+BCGD &  &  & {\bf 88.74} & {\bf 90.08} & 91.38 \\
  \hline
  ResNet-20 + BC & \multirow{2}{*}{92.41} & \multirow{2}{*}{91.90} & 89.23 & 90.89 & 91.53 \\
  \cline{1-1}\cline{4-6}
  ResNet-20+BCGD &  &  &  {\bf 90.10}  & {\bf 91.15}& 91.56 \\
  \hline
\end{tabular}
\caption{CIFAR-10 validation accuracies in \% with the a.e. $\bm{\alpha}$ derivative.}
\label{tab:1}
\end{table}

\begin{figure}[ht]
\centering
\begin{tabular}{cc}
\includegraphics[width=0.48\textwidth]{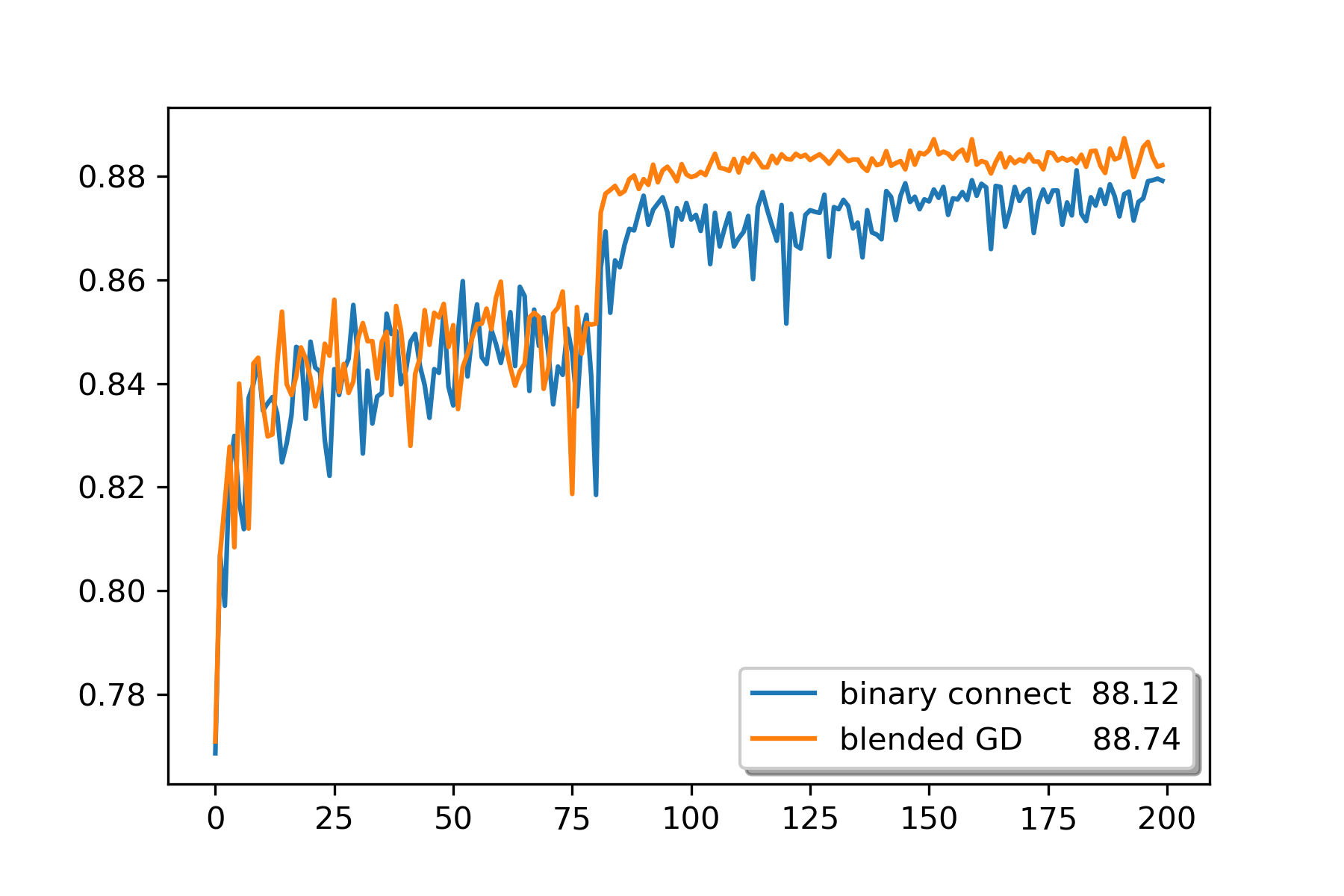}
\includegraphics[width=0.48\textwidth]{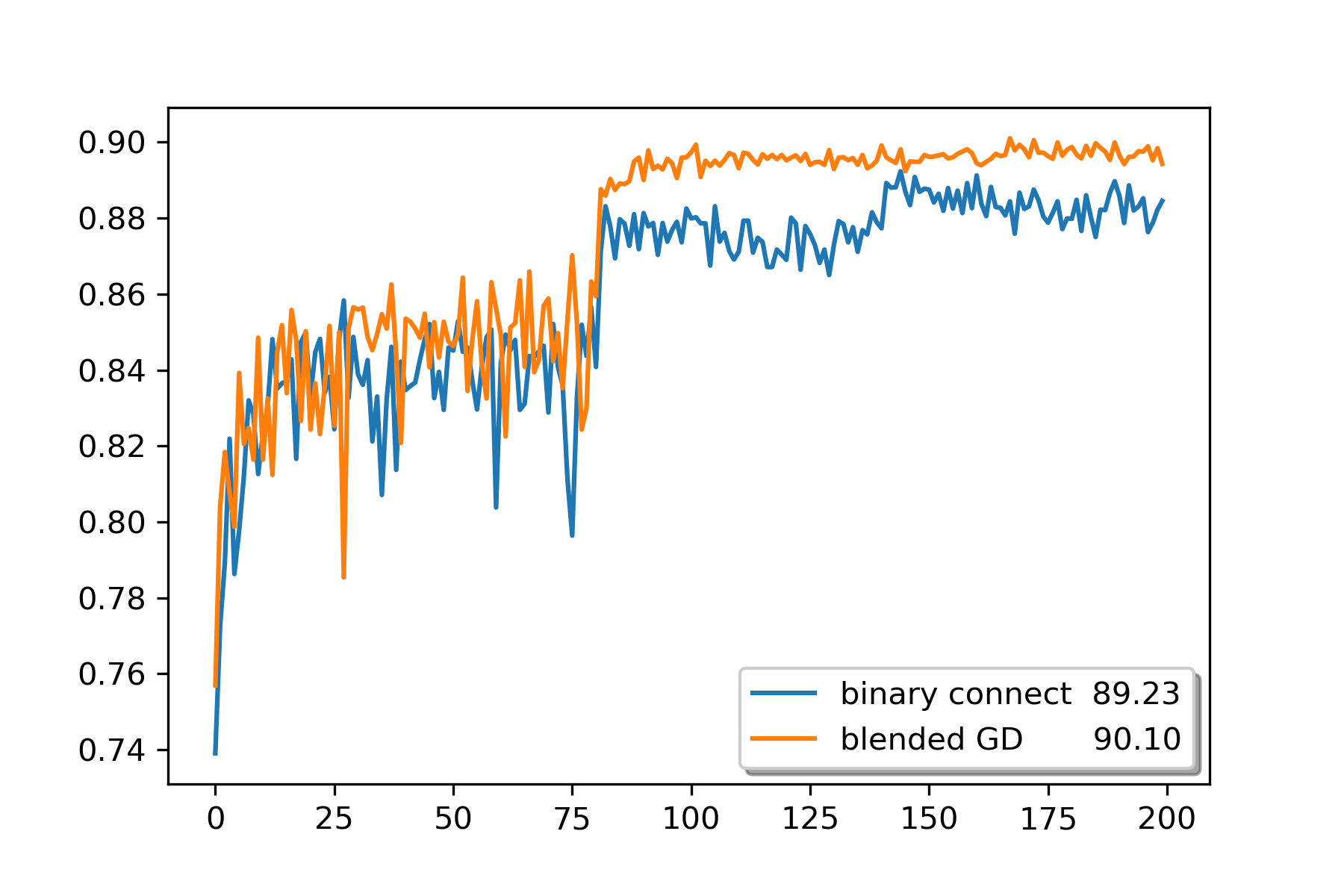}
\end{tabular}
\caption{CIFAR-10 validation accuracies vs. epoch numbers with a.e. $\bm{\alpha}$ derivative and 1W4A quantization on VGG-11 (left) and ResNet-20 (right), with (orange) and without (blue) blending which speeds up training towards higher accuracies.}\label{fig:1}
\end{figure}

\begin{table}[ht]
\centering
\begin{tabular}{|c|c|c|c|c|c|c|c|}
  \hline			
  Network & Float & 32W4A & 1W4A & 2W4A & 4W4A  \\
  \hline
  VGG-11 + BC & \multirow{2}{*}{92.13} & \multirow{2}{*}{92.08} & 89.12 & 90.52 & {\bf 91.89} \\
  \cline{1-1}\cline{4-6}
  VGG-11+BCGD &  &  & {\bf 89.59} & {\bf 90.71} & 91.70 \\
  \hline
  ResNet-20 + BC & \multirow{2}{*}{92.41} & \multirow{2}{*}{92.14} & 89.37 & 91.02 & 91.71 \\
  \cline{1-1}\cline{4-6}
  ResNet-20+BCGD &  &  & {\bf 90.05} & 91.03 & {\bf 91.97} \\
  \hline
\end{tabular}
\caption{CIFAR-10 validation accuracies with the 3-valued $\bm{\alpha}$ derivative.}
\label{tab:2}
\end{table}

\begin{table}[ht]
\centering
\begin{tabular}{|c|c|c|c|c|c|c|c|}
  \hline			
 Network & Float & 32W4A & 1W4A & 2W4A & 4W4A  \\
  \hline
  VGG-11 + BC & \multirow{2}{*}{92.13} & \multirow{2}{*}{91.66} & 88.50 & 89.99 & 91.31 \\
  \cline{1-1}\cline{4-6}
  VGG-11+BCGD &  &  & {\bf 89.12} & 90.00 & 91.31 \\
  \hline
  ResNet-20 + BC & \multirow{2}{*}{92.41} & \multirow{2}{*}{91.73} & 89.22 & 90.64 & 91.37 \\
  \cline{1-1}\cline{4-6}
  ResNet-20+BCGD &  &  & {\bf 89.98} & {\bf 90.75} & {\bf 91.65} \\
  \hline
\end{tabular}
\caption{CIFAR-10 validation accuracies with the 2-valued $\bm{\alpha}$ derivative (PACT \cite{pact}).}
\label{tab:3}
\end{table}

ImageNet (ILSVRC12) dataset \cite{imagenet_09} is a benchmark for large-scale image classification task, which has $1.2$ million images for training and $50,000$ for validation of 1,000 categories. We set mini-batch size to $256$ and trained the models for 80 epochs with learning rate decaying at epoch 50 and 70. The weight decay of $10^{-5}$ and momentum of $0.9$ were used. The ResNet-18 accuracies 65.46\%/86.36\% at 1W4A in Table 4 outperformed HWGQ \cite{halfwave_17} where top-1/top-5 accuracies are 60.8\%/83.4\% with non-quantized first/last convolutional layers. The results in the Table 4 and Table 5 show that using the 3-valued coarse $\bm{\alpha}$ partial derivative appears more effective than the 2-valued as quantization bit precision is lowered. We also observe that the accuracies degrade gracefully from 4W8A to 1W4A for ResNet-18 while quantizing all convolutional layers. Again, BCGD converges much faster than BC towards higher accuracy as illustrated by Figure \ref{fig:2}.

\begin{table}[ht]
\centering
\begin{tabular}{|c|c|c|c|c|c|c|c|}
  \hline			
 \multirow{2}{*}{} & \multirow{2}{*}{Float} & \multicolumn{2}{c|}{1W4A} & \multicolumn{2}{c|}{4W4A} & \multicolumn{2}{c|}{4W8A} \\
  \cline{3-8}
  & & 3 valued & 2 valued & 3 valued & 2 valued & 3 valued & 2 valued \\
  \hline
  top-1 & 69.64 & 64.36/$65.46^*$ & 63.37/$64.57^*$ & 67.36 & 66.97 & 68.85 & 68.83 \\
  \hline
  top-5 & 88.98 & 85.65/$86.36^*$ & 84.93/$85.75^*$ & 87.76 & 87.41 & 88.71 & 88.84\\
  \hline
\end{tabular}
\caption{ImageNet validation accuracies with BCGD on ResNet-18. Starred accuracies are with first and last convolutional layers in float precision as in \cite{halfwave_17}. The accuracies are for quantized weights across all layers otherwise.}
\label{tab:4}
\end{table}

\begin{table}[ht]
\centering
\begin{tabular}{|c|c|c|c|c|c|c|c|}
  \hline			
 \multirow{2}{*}{} & \multirow{2}{*}{Float} & \multicolumn{2}{c|}{1W4A} & \multicolumn{2}{c|}{4W4A} & \multicolumn{2}{c|}{4W8A} \\
  \cline{3-8}
  & & 3 valued & 2 valued & 3 valued & 2 valued & 3 valued & 2 valued \\
  \hline
  top-1 & 73.27 & 68.43 & 67.51 & 70.81 & 70.01 & 72.07 & 72.18 \\
  \hline
  top-5 & 91.43 & 88.29 & 87.72 & 90.00 & 89.49 & 90.71 & 90.73\\
  \hline
\end{tabular}
\caption{ImageNet validation accuracies with BCGD on ResNet-34. The accuracies are for quantized weights across all layers.}
\label{tab:5}
\end{table}

\begin{figure}[H]
\centering
\begin{tabular}{cc}
\includegraphics[width=0.48\textwidth]{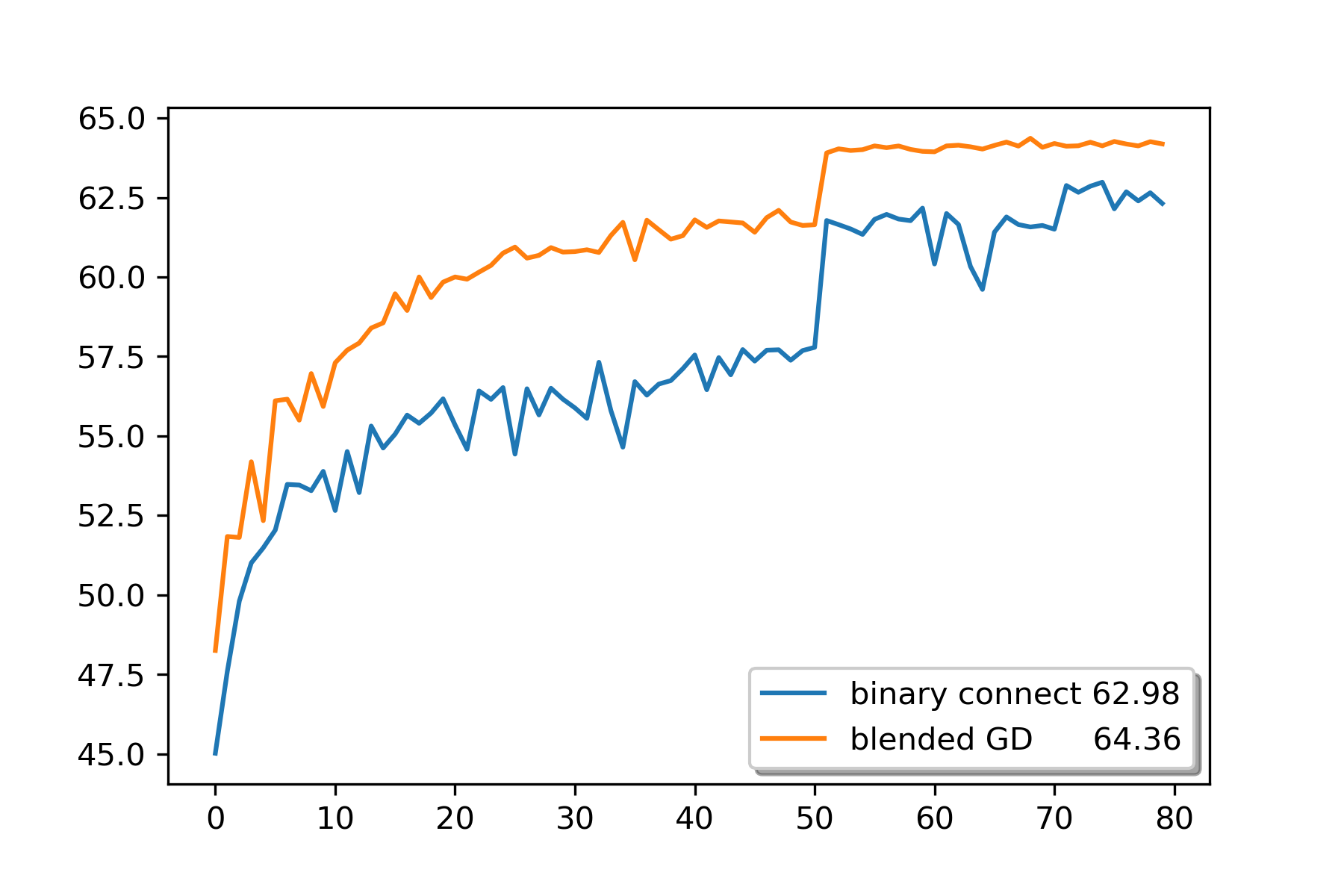}
\includegraphics[width=0.48\textwidth]{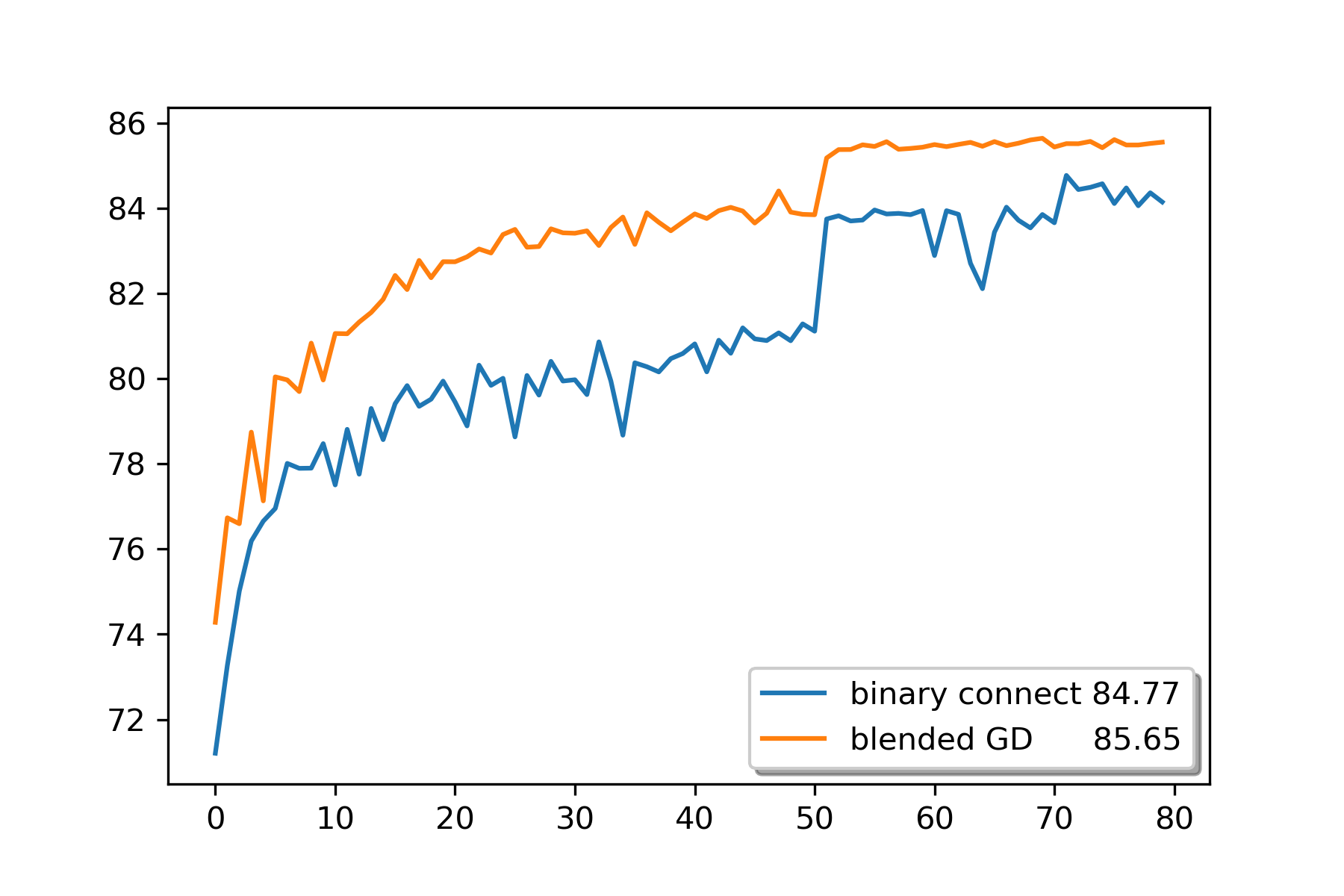}
\end{tabular}
\caption{ImageNet validation accuracies (left: top-1, right: top-5) vs. number of epochs with 3-valued derivative on 1W4A quantization on ResNet-18 with (orange) and without (blue) blending which substantially speeds up training towards higher accuracies.}\label{fig:2}
\end{figure}

\medskip

\section{Analysis of Coarse Gradient Descent for Activation Quantization}
\label{sec:analy}
As a proof of concept, we analyze a simple two-layer network with binarized ReLU activation. Let $\sigma$ be the binarized ReLU function, same as hard threshold activation \cite{hinton2012neural}, with the bit-width $b_a=1$ and the resolution $\alpha \equiv 1$ in (\ref{eq:qrelu}):
\begin{equation*}
\sigma(x) = 
\begin{cases}
0 \quad \mbox{if } x \leq 0, \\
1 \quad \mbox{if } x >0 .
\end{cases}
\end{equation*}
We define the training sample loss by
$$
\ell(\v,\w; \Z): = \frac{1}{2}\Big(\v^\t \sigma(\Z\w) - (\v^*)^\t \sigma(\Z\w^*)\Big)^2,
$$
where $\v^*\in\R^m$ and $\w^*\in\R^n$ are the underlying (nonzero) teacher parameters in the second and first layers, respectively. Same as in the literature that analyze the conventional ReLU nets \cite{Lee,li2017convergence,tian2017analytical,brutzkus2017globally}, we assume the entries of $\Z\in\R^{m\times n}$ are i.i.d. sampled from the standard normal distribution $\mathcal{N}(0,1)$. Note that $\ell(\v,\w;\Z) = \ell(\v,\w/c; \Z)$ for any scalar $c>0$. Without loss of generality, we fix $\|\w^*\| = 1$.

\subsection{Population Loss Minimization}
Suppose we have $N$ independent training samples $\{\Z^{(1)},\dots, \Z^{(N)}\}$, then the associated empirical risk minimization reads
\begin{equation}\label{eq:mean_model}
\min_{\v\in\R^m,\w\in\R^n} \; \frac{1}{N}\sum_{i=1}^N \ell(\v,\w; \Z^{(i)}).
\end{equation}
The major difficulty of analysis here is that the empirical risk function in (\ref{eq:mean_model}) is still piecewise constant and has a.e. zero partial $\w$ gradient. This issue can be resolved by instead considering the following population loss minimization \cite{li2017convergence,brutzkus2017globally,Lee,tian2017analytical}:
\begin{equation}\label{eq:model}
\min_{\v\in\R^m,\w\in\R^n} \; f(\v, \w) := \E_\Z \left[\ell(\v,\w; \Z)\right].
\end{equation}
Specifically, in the limit $N\to\infty$, the objective function $f$ becomes favorably smooth with non-trivial gradient. 
For nonzero vector $\w$, let us define the angle between $\w$ and $\w^*$ by 
$$
\theta(\w,\w^*) := \arccos\Big(\frac{\w^\t\w^*}{\|\w\|\|\w^*\|}\Big) = \arccos\Big(\frac{\w^\t\w^*}{\|\w\|}\Big),
$$
then we have 
\begin{lem}\label{lem:obj}
If every entry of $\Z$ is i.i.d. sampled from $\mathcal{N}(0,1)$, $\|\w^*\|=1$, and $\|\w\|\neq 0$, then the population loss is
\begin{equation}\label{eq:loss}
f(\v,\w) = \frac{1}{8}\left[\v^\t\big(\I + \1\1^\t \big) \v -2\v^\t \left( \left(1-\frac{2}{\pi}\theta(\w,\w^*) \right)\I + \1\1^\t \right)\v^*  + (\v^*)^\t \big(\I + \1\1^\t \big)\v^* \right].
\end{equation}
Moreover, the gradients of $f(\v,\w)$ w.r.t. $\v$ and $\w$ are
\begin{equation}\label{eq:grad_v}
\frac{\partial f}{\partial \v}(\v,\w) = \frac{1}{4}\big(\I + \1\1^\t \big) \v - \frac{1}{4}\left( \left(1-\frac{2}{\pi}\theta(\w,\w^*) \right)\I + \1\1^\t \right)\v^* 
\end{equation}
and
\begin{equation}\label{eq:grad_w}
\frac{\partial f}{\partial \w}(\v,\w) = -\frac{\v^\t\v^*}{2\pi\|\w\|}
\frac{\Big(\I - \frac{\w\w^\t}{\|\w\|^2}\Big)\w^*}{\Big\| \Big(\I - \frac{\w\w^\t}{\|\w\|^2}\Big)\w^*\Big\|}, \quad \mbox{for } \theta(\w,\w^*)\in(0, \pi), 
\end{equation}
respectively.
\end{lem}

\medskip

When $\w \neq \0$, the possible (local) minimizers of problem (\ref{eq:model}) are located at
\begin{enumerate}
\item Stationary points where the gradients defined in (\ref{eq:grad_v}) and (\ref{eq:grad_w}) vanish simultaneously (which may not be possible), i.e., 
\begin{equation}\label{eq:critical}
\v^\t\v^* = 0 \mbox{ and } \v = \big(\I + \1\1^\t \big)^{-1}\left( \left(1-\frac{2}{\pi}\theta(\w,\w^*) \right)\I + \1\1^\t \right)\v^* .
\end{equation}
\item Non-differentiable points where $\theta(\w,\w^*) = 0$ and $\v = \v^*$, or $\theta(\w,\w^*) = \pi$ and $\v = \big(\I + \1\1^\t \big)^{-1}( \1\1^\t -\I )\v^*$. 
\end{enumerate}
Among them, $\{(\v,\w): \v = \v^*, \, \theta(\w,\w^*) = 0\}$ are the global minimizers with $f(\v,\w) = 0$.

\begin{prop}\label{prop}
If $(\1^\t\v^*)^2 < \frac{m+1}{2}\|\v^*\|^2$, then 
\begin{align*}
\bigg\{(\v,\w) \in\R^{m+n}:  \v = (\I + \1\1^\t)^{-1} & \left(\frac{-(\1^\t\v^*)^2}{(m+1)\|\v^*\|^2- (\1^\t\v^*)^2}\I + \1\1^\t\right)\v^*, \\
& \qquad \qquad \qquad \theta(\w,\w^*) = \frac{\pi}{2}\frac{(m+1)\|\v^*\|^2}{(m+1)\|\v^*\|^2 - (\1^\t\v^*)^2} \bigg\}
\end{align*}
gives the stationary points obeying (\ref{eq:critical}). Otherwise, problem (\ref{eq:model}) has no stationary points.  
\end{prop}

The gradient of the population loss, $\left(\frac{\partial f}{\partial \v}, \, \frac{\partial f}{\partial \w}\right)(\v,\w)$, holds Lipschitz continuity under a boundedness condition.

\begin{lem}\label{lem:lipschitz}
For any $(\v,\w)$ and $(\tv,\tw)$ with $\min\{ \|\w\|, \,  \|\tw\|\} = c>0$ and $\max\{\|\v\|, \, \|\tv\|\} = C$, there exists a constant $L>0$ depending on $c$ and $C$, such that 
$$
\left\|\left( \frac{\partial f}{\partial \v}, \frac{\partial f}{\partial \w} \right)(\v,\w) - \left(\frac{\partial f}{\partial \v}, \frac{\partial f}{\partial \w} \right)(\tv,\tw) \right\| \leq L \|(\v, \w) - (\tv, \tw)\|.
$$
\end{lem}

\subsection{Convergence Analysis of Normalized Coarse Gradient Descent}
The partial gradients $\frac{\partial f}{\partial \v}$ and $\frac{\partial f}{\partial \w}$, however, are not available in the training. What we really have access to are the expectations of the sample gradients, namely, 
$$\E_\Z\left[\frac{\partial \ell}{\partial \v}(\v,\w;\Z)\right] \mbox{ and } \E_\Z\left[\frac{\partial \ell}{\partial \w}(\v,\w;\Z)\right]. $$
If $\sigma$ was differentiable, then the back-propagation reads
\begin{equation}\label{eq:bp_v}
\frac{\partial \ell}{\partial \v}(\v,\w;\Z) = \sigma(\Z\w)\Big(\v^\t \sigma(\Z\w) - (\v^*)^\t \sigma(\Z\w^*)\Big).
\end{equation}
and
\begin{equation}\label{eq:bp_w}
\frac{\partial \ell}{\partial \w}(\v,\w;\Z) = \Z^\t\big(\sigma^{\prime}(\Z\w)\odot\v\big)\Big(\v^\t \sigma(\Z\w) - (\v^*)^\t \sigma(\Z\w^*)\Big).
\end{equation}
Now that $\sigma$ has zero derivative a.e., which makes (\ref{eq:bp_w}) inapplicable. We study the coarse gradient descent with $\sigma^{\prime}$ in (\ref{eq:bp_w}) being replaced by the (sub)derivative $\mu^{\prime}$ of regular ReLU $\mu(x):= \max(x,0)$. More precisely, we use the following surrogate of $\frac{\partial \ell}{\partial \w}(\v,\w;\Z)$:
\begin{equation}\label{eq:cbp_w}
\g(\v,\w;\Z) = \Z^\t\big(\mu^{\prime}(\Z\w)\odot\v\big)\Big(\v^\t \sigma(\Z\w) - (\v^*)^\t \sigma(\Z\w^*)\Big)
\end{equation}
with $\mu'(x) = \sigma(x)$, and consider the following coarse gradient descent with weight normalization:
\begin{equation}\label{eq:iter}
\begin{cases}
\v^{t+1} = \v^t - \eta \E_\Z  \left[\frac{\partial \ell}{\partial \v}(\v^t,\w^t;\Z)\right] \\
\w^{t+\frac{1}{2}} = \w^t - \eta \E_\Z \left[\g(\v^t,\w^t; \Z)\right] \\
\w^{t+1} = \frac{\w^{t+1/2}}{\left\|\w^{t+1/2}\right\|}
\end{cases}
\end{equation}

\medskip

\begin{lem} \label{lem:cgd}
The expected gradient of $\ell(\v,\w;\Z)$ w.r.t. $\v$ is
\begin{equation}\label{eq:cgrad_v}
\E_\Z  \left[\frac{\partial \ell}{\partial \v}(\v,\w;\Z)\right] = \frac{\partial f}{\partial \v}(\v,\w) = \frac{1}{4}\big(\I + \1\1^\t \big) \v - \frac{1}{4}\left( \left(1-\frac{2}{\pi}\theta(\w,\w^*) \right)\I + \1\1^\t \right)\v^* .
\end{equation}
The expected coarse gradient w.r.t. $\w$ is 
\begin{equation}\label{eq:cgrad_w}
\E_\Z \Big[\g(\v,\w; \Z)\Big] = 
\frac{h(\v,\v^*)}{2\sqrt{2\pi}}\frac{\w}{\|\w\|} - \cos\left(\frac{\theta(\w,\w^*)}{2}\right)\frac{\v^\t\v^*}{\sqrt{2\pi}}\frac{\frac{\w}{\|\w\|} + \w^*}{\left\|\frac{\w}{\|\w\|} + \w^*\right\|},\footnote{We redefine the second term as $\0$ in the case $\theta(\w,\w^*) = \pi$.}  
\end{equation}
where $h(\v,\v^*) = \|\v\|^2+ (\1^\t\v)^2 - (\1^\t\v)(\1^\t\v^*) + \v^\t\v^*$. In particular, $\E_\Z \Big[\frac{\partial \ell}{\partial \v}(\v,\w; \Z)\Big]$ and $\E_\Z \Big[\g(\v,\w; \Z)\Big]$ vanish simultaneously only in one of the following cases
\begin{enumerate}
\item (\ref{eq:critical}) is satisfied according to Proposition \ref{prop}.
\item $\v = \v^*$, $\theta(\w,\w^*)=0$, or $\v = (\I + \1\1^\t)^{-1}(\1\1^\t - \I)\v^*$, $\theta(\w,\w^*)=\pi$.
\end{enumerate}
\end{lem}

\medskip

What is interesting is that the coarse partial gradient $\E_\Z \Big[\g(\v,\w; \Z)\Big]=\0$ is properly defined at global minimizers of the population loss minimization problem (\ref{eq:model}) with $\v = \v^*$, $\theta(\w,\w^*)=0$, whereas the true gradient $\frac{\partial f}{\partial \w}(\v,\w)$ does not exist there. Our key finding is that the coarse gradient \emph{$\E_\Z \Big[\g(\v,\w; \Z)\Big]$ has positive correlation with the true gradient $\frac{\partial f}{\partial \w}(\v,\w) $, and consequently, $-\E_\Z \Big[\g(\v,\w; \Z)\Big]$ together with $-\E_\Z  \left[\frac{\partial \ell}{\partial \v}(\v,\w;\Z)\right]$ give a descent direction in algorithm (\ref{eq:iter}).} 
\begin{lem}\label{cor}
If $\theta(\w,\w^*)\in(0, \pi)$ , and $\|\w\|\neq 0$, then the inner product between the expected coarse and true gradients w.r.t. $\w$ is
\begin{equation*}
\left\langle \E_\Z \Big[\g(\v,\w; \Z)\Big], \frac{\partial f}{\partial \w}(\v,\w) \right\rangle  = 
\frac{\sin\left(\theta(\w,\w^*)\right)}{2(\sqrt{2\pi})^3\|\w\|}(\v^\t\v^*)^2 \geq 0.
\end{equation*}
\end{lem}

Moreover, the following lemma asserts that $\E_\Z \Big[\g(\v,\w; \Z)\Big]$ is sufficiently correlated with $\frac{\partial f}{\partial \w}(\v,\w)$, which will secure sufficient descent in objective values $\{f(\v^t,\w^t)\}$ and thus the convergence of $\{(\v^t,\w^t)\}$.

\begin{lem}\label{lem:correlated}
Suppose $\|\w\| = 1$ and $\|\v\|\leq C$. There exists a constant $A>0$ depending on $C$, such that
$$
\left\|\E_\Z \Big[\g(\v,\w; \Z)\Big] \right\|^2 
\leq  
A\left(\left\|\frac{\partial f}{\partial \v}(\v,\w)\right\|^2 + \left\langle \E_\Z \Big[\g(\v,\w; \Z)\Big], \frac{\partial f}{\partial \w}(\v,\w) \right\rangle \right).
$$
\end{lem}

\medskip

Equipped with Lemma \ref{lem:lipschitz} and Lemma \ref{lem:correlated}, we are able to show the convergence result of iteration (\ref{eq:iter}).
\begin{thm}\label{thm}
Given the initialization $(\v^0,\w^0)$ with $\|\w^0\|=1$, and let $\{(\v^t, \w^t)\}$ be the sequence generated by iteration (\ref{eq:iter}). There exists $\eta_0>0$, such that for any step size $\eta<\eta_0$, $\{f(\v^t,\w^t)\}$ is monotonically decreasing, and both $\left\|\E_\Z \Big[\frac{\partial \ell}{\partial \v}(\v^t,\w^t; \Z)\Big]\right\|$ and $\left\|\E_\Z \Big[\g(\v^t,\w^t; \Z)\Big]\right\|$ converge to 0, as $t\to\infty$. 
\end{thm}

\begin{rem}
Combining the treatment of \cite{Lee} for analyzing two-layer networks with regular ReLU and the positive correlation between $\E_\Z\left[ \g(\w,\v;\Z)\right]$ and $\frac{\partial f}{\partial \w}(\v,\w)$, one can further show that if the initialization $(\v^0,\w^0)$ satisfies $(\v^0)^\t\v^*>0$, $\theta(\w^0,\w^*)<\frac{\pi}{2}$ and $(\1^\t \v^*)(\1^\t \v^0)\leq (\1^\t\v^*)^2$, then $\{(\v^t,\w^t)\}$ converges to the global minimizer $(\v^*,\w^*)$.
\end{rem}

\section{Concluding Remarks}
We introduced the concept of coarse gradient for activation quantization problem of DNNs, for which the a.e. gradient is inapplicable. Coarse gradient is generally not a gradient but an artificial ascent direction. We further proposed BCGD algorithm, for training fully quantized neural networks. The weight update of  BCGD goes by coarse gradient correction of a weighted average of the float weights and their quantization, which yields sufficient descent in objective and thus acceleration. Our experiments demonstrated that BCGD is very effective for quantization at extremely low bit-width such as binarization. Finally, we analyzed the coarse gradient descent for a two-layer neural network model with Gaussian input data, and proved that the expected coarse gradient essentially correlates positively with the underlying true gradient.

\bigskip

\noindent{\bf Acknowledgements.}
This work was partially supported by NSF grants DMS-1522383, IIS-1632935; ONR grant N00014-18-1-2527, AFOSR grant FA9550-18-0167, 
DOE grant DE-SC0013839 and STROBE STC NSF grant DMR-1548924.




\section*{Conflict of Interest Statement}

On behalf of all authors, the corresponding author states that there is no conflict of interest.


\bibliographystyle{spmpsci}  
\bibliography{references}

\clearpage

\section*{\large Appendix}

\subsection*{\bf A. Additional Preliminaries}
\begin{lem}\label{lem:basic}
Let $\z$ be a Gaussian random vector with entries i.i.d. sampled from $\mathcal{N}(0,1)$. Given nonzero vectors $\w$ and $\tw$ with angle $\theta$, we have
\begin{align*}
\E\left[1_{\{\z^\t\w> 0\}}\right] = \frac{1}{2}, \; \E\left[1_{\{\z^\t\w> 0, \, \z^\t\tw> 0\}}\right] = \frac{\pi-\theta}{2\pi},
\end{align*}
and
\begin{equation*}
\E\left[\z 1_{\{\z^\t\w >0\} }   \right] = \frac{1}{\sqrt{2\pi}} \frac{\w}{\|\w\|}, \;
\E\left[ \z 1_{\{\z^\t \w>0, \, \z^\t \w^* >0\}}\right] = 
\frac{\cos(\theta/2)}{\sqrt{2\pi}} \frac{\frac{\w}{\|\w\|} + \frac{\tw}{\|\tw\|} }{\left\|\frac{\w}{\|\w\|} + \frac{\tw}{\|\tw\|}  \right\|}. \footnote{Same as in Lemma \ref{lem:cgd}, we redefine $\E\left[ \z 1_{\{\z^\t \w>0, \, \z^\t \w^* >0\}}\right]=\0$ in the case $\theta(\w,\w^*) = \pi$.}
\end{equation*}
\end{lem}

\begin{proof}
The third identity was proved in Lemma A.1 of \cite{Lee}. To show the first one, since Gaussian distribution is rotation-invariant, without loss of generality we assume $\w = [w_1,0,\0^\t]^\t$ with $w_1> 0$, then  $\E\left[1_{\{\z^\t\w> 0\}}\right] = \P(z_1>0) = \frac{1}{2}$. 

\medskip

We further assume $\tw = [\tilde{w}_1,\tilde{w}_2,\0^\t]^\t$. It is easy to see 
$$
\E\left[1_{\{\z^\t\w> 0, \, \z^\t\tw>0\}}\right] = \P(\z^\t\w> 0, \, \z^\t\tw>0) = \frac{\pi- \theta}{2\pi},
$$
which is the probability that $\z$ forms an acute angle with both $\w$ and $\w^*$.
\medskip

To prove the last identity, we use polar representation of 2-D Gaussian random variables, where $r$ is the radius and $\phi$ is the angle with $\mathrm{d}\P_r = r \exp(-r^2/2)\mathrm{d}r$
and $\mathrm{d}\P_\phi = \frac{1}{2\pi}\mathrm{d}\phi$. Then
$\E\left[ z_i 1_{\{\z^\t \w>0, \, \z^\t \w^* >0\}}\right] = 0$ for $i\geq 3$. Moreover,
$$
\E\left[ z_1 1_{\{\z^\t \w>0, \, \z^\t \w^* >0\}}\right] = \frac{1}{2\pi}\int_{0}^\infty r^2\exp\left(-\frac{r^2}{2}\right) \mathrm{d}r \int_{-\frac{\pi}{2}+\theta}^{\frac{\pi}{2}} \cos(\phi) \mathrm{d}\phi = \frac{1+\cos(\theta)}{2\sqrt{2\pi}}
$$
and
$$
\E\left[ z_2 1_{\{\z^\t \w>0, \, \z^\t \w^* >0\}}\right] = \frac{1}{2\pi}\int_{0}^\infty r^2\exp\left(-\frac{r^2}{2}\right) \mathrm{d}r \int_{-\frac{\pi}{2}+\theta}^{\frac{\pi}{2}} \sin(\phi) \mathrm{d}\phi = \frac{\sin(\theta)}{2\sqrt{2\pi}}.
$$
Therefore, 
$$
\E\left[ \z 1_{\{\z^\t \w>0, \, \z^\t \w^* >0\}}\right] = \frac{\cos(\theta/2)}{\sqrt{2\pi}}[\cos(\theta/2), \sin(\theta/2),\0^\t]^\t = \frac{\cos(\theta/2)}{\sqrt{2\pi}} \frac{\frac{\w}{\|\w\|} + \frac{\tw}{\|\tw\|} }{\left\|\frac{\w}{\|\w\|} + \frac{\tw}{\|\tw\|}  \right\|},
$$ 
where the last equality holds because $\frac{\w}{\|\w\|}$ and $\frac{\tw}{\|\tw\|}$ are two unit-normed vectors with angle $\theta$.

\end{proof}

\bigskip

\begin{lem}\label{lem:angle}
For any nonzero vectors $\w$ and $\tw$ with $\|\tw\|\geq  \|\w\| = c>0$, we have 
\begin{enumerate}
\item $|\theta(\w,\w^*)-\theta(\tw,\w^*)|\leq \frac{\pi}{2c}\|\w - \tw\|$.
\item $\left\|\frac{1}{\|\w\|}
\frac{\Big(\I - \frac{\w\w^\t}{\|\w\|^2}\Big)\w^*}{\Big\| \Big(\I - \frac{\w\w^\t}{\|\w\|^2}\Big)\w^*\Big\|}  -
\frac{1}{\|\tw\|}
\frac{\Big(\I - \frac{\tw\tw^\t}{\|\tw\|^2}\Big)\w^*}{\Big\| \Big(\I - \frac{\tw\tw^\t}{\|\tw\|^2}\Big)\w^*\Big\|} \right\|  \leq \frac{1}{c^2}\|\w - \tw\|$.
\end{enumerate}
\end{lem}

\begin{proof}
1. Since by Cauchy-Schwarz inequality,
$$
\left\la \tw , \w - \frac{c\tw}{\|\tw\|}\right \ra = \tw^\t \w - c\|\tw\| \leq 0,
$$
we have
\begin{align}\label{eq:4}
\|\tw - \w\|^2 = & \; \left\|\left( 1-\frac{c}{\|\tw\|}  \right)\tw  - \left(\w -\frac{c\tw}{\|\tw\|}  \right) \right\|^2 \geq \left\| \left( 1-\frac{c}{\|\tw\|}  \right)\tw \right\|^2 + \left\| \w -\frac{c\tw}{\|\tw\|} \right\|^2 \notag \\
\geq & \; \left\| \w -\frac{c\tw}{\|\tw\|} \right\|^2  = c^2 \left\|\frac{\w}{\|\w\|} - \frac{\tw}{\|\tw\|}\right\|^2.
\end{align}
Therefore,
\begin{align*}
& \; |\theta(\w,\w^*)-\theta(\tw,\w^*)| \leq \theta(\w,\tw) = \theta\left(\frac{\w}{\|\w\|},\frac{\tw}{\|\tw\|}\right) \\
\leq & \; \pi\sin \left(\frac{\theta\left(\frac{\w}{\|\w\|},\frac{\tw}{\|\tw\|}\right)}{2}\right) = \frac{\pi}{2}\left\|\frac{\w}{\|\w\|} - \frac{\tw}{\|\tw\|}\right\|
\leq \frac{\pi}{2c}\|\w - \tw\|,
\end{align*}
where we used the fact $\sin(x)\geq \frac{2x}{\pi}$ for $x\in [0,\frac{\pi}{2}]$ and the estimate in (\ref{eq:4}).\\

2. Since $\Big(\I - \frac{\w\w^\t}{\|\w\|^2}\Big)\w^*$ is the projection of $\w^*$ onto the complement space of $\w$, and likewise for $\Big(\I - \frac{\tw\tw^\t}{\|\tw\|^2}\Big)\w^*$, the angle between  $\Big(\I - \frac{\w\w^\t}{\|\w\|^2}\Big)\w^*$ and $\Big(\I - \frac{\tw\tw^\t}{\|\tw\|^2}\Big)\w^*$ is equal to the angle between $\w$ and $\tw$. Therefore,
$$
\left\la \frac{\Big(\I - \frac{\w\w^\t}{\|\w\|^2}\Big)\w^*}{\Big\| \Big(\I - \frac{\w\w^\t}{\|\w\|^2}\Big)\w^*\Big\|} , 
\frac{\Big(\I - \frac{\tw\tw^\t}{\|\tw\|^2}\Big)\w^*}{\Big\| \Big(\I - \frac{\tw\tw^\t}{\|\tw\|^2}\Big)\w^*\Big\|} \right\ra = \left\la \frac{\w}{\|\w\|} , \frac{\tw}{\|\tw\|} \right\ra,
$$
and thus
\begin{align*}
\left\|\frac{1}{\|\w\|}
\frac{\Big(\I - \frac{\w\w^\t}{\|\w\|^2}\Big)\w^*}{\Big\| \Big(\I - \frac{\w\w^\t}{\|\w\|^2}\Big)\w^*\Big\|}  -
\frac{1}{\|\tw\|}
\frac{\Big(\I - \frac{\tw\tw^\t}{\|\tw\|^2}\Big)\w^*}{\Big\| \Big(\I - \frac{\tw\tw^\t}{\|\tw\|^2}\Big)\w^*\Big\|} \right\|  
= \left\| \frac{\w}{\|\w\|^2} - \frac{\tw}{\|\tw\|^2} \right\| 
= \frac{\|\w - \tw \|}{\|\w\|\|\tw\|}\leq \frac{1}{c^2}\|\w - \tw\|.
\end{align*}
The second equality above holds because
$$ 
\left\| \frac{\w}{\|\w\|^2} - \frac{\tw}{\|\tw\|^2} \right\|^2 
= \frac{1}{\|\w\|^2} + \frac{1}{\|\tw\|^2} - \frac{2\la \w, \tw \ra}{\|\w\|^2 \|\tw\|^2} = \frac{\|\w - \tw\|^2}{\|\w\|^2 \|\tw\|^2}.
$$
\end{proof}

\subsection*{\bf B. Proofs}
\begin{proof}[{\bf Proof of Proposition \ref{prop:bgd}}]
We rewrite the update (\ref{bcgd}) as
$$
\w^{t+1} = \arg\min_{\w\in\Q} \; \langle \w, \nabla f(\w^t) \rangle + \frac{1-\rho}{2\eta} \|\w-\w_f^t\|^2 + \frac{\rho}{2\eta} \|\w-\w^t\|^2 .
$$
Then since $\w^t, \, \w^{t+1} \in\Q$, we have
$$
\langle \w^{t+1}, \nabla f(\w^t) \rangle + \frac{1-\rho}{2\eta} \|\w^{t+1}-\w_f^t\|^2 + \frac{\rho}{2\eta} \|\w^{t+1}-\w^t\|^2 \leq \langle \w^t, \nabla f(\w^t) \rangle + \frac{1-\rho}{2\eta} \|\w^t-\w_f^t\|^2,
$$
or equivalently,
\begin{equation}\label{eq:1}
\langle \w^{t+1}-\w^t, \nabla f(\w^t) \rangle + \frac{1-\rho}{2\eta} \left(\|\w^{t+1}-\w_f^t\|^2-\|\w^t-\w_f^t\|^2 \right) + \frac{\rho}{2\eta}\|\w^{t+1}-\w^t\|^2  \leq  0.
\end{equation}
On the other hand, since $f$ has $L$-Lipschitz gradient, the descent lemma \cite{bertsekas1999nonlinear} gives
\begin{equation}\label{eq:2}
f(\w^{t+1})\leq f(\w^t) + \la \nabla f(\w^t), \w^{t+1}-\w^t \ra + \frac{L}{2}\|\w^{t+1}-\w^t\|^2.
\end{equation}
Combining (\ref{eq:1}) and (\ref{eq:2}) completes the proof.
\end{proof}

\bigskip

\begin{proof}[{\bf Proof of Lemma \ref{lem:obj}}]
We first evaluate $\E_\Z\left[\sigma(\Z\w)\sigma(\Z\w)^\t\right]$, $\E_\Z\left[\sigma(\Z\w)\sigma(\Z\w^*)^\t\right]$, and $\E_\Z\left[\sigma(\Z\w^*)\sigma(\Z\w^*)^\t\right]$. Let $\Z_i^\t$ be the $i$-th row vector of $\Z$. Since $\w\neq\0$, using Lemma \ref{lem:basic}, we have
$$
\E_\Z\left[\sigma(\Z\w)\sigma(\Z\w)^\t\right]_{ii} = \E\left[\sigma(\Z_i^\t\w)\sigma(\Z_i^\t\w)\right] = \E\left[1_{\{\Z_i^\t\w> 0\}}\right] = \frac{1}{2},
$$
and for $i\neq j$,
$$
\E_\Z\left[\sigma(\Z\w)\sigma(\Z\w)^\t\right]_{ij} = \E\left[\sigma(\Z_i^\t\w)\sigma(\Z_j^\t\w)\right] = \E\left[1_{\{\Z_i^\t\w> 0\}}\right]\E\left[1_{\{\Z_j^\t\w> 0\}}\right] = \frac{1}{4}.
$$
Therefore, $\E_\Z\left[\sigma(\Z\w)\sigma(\Z\w)^\t\right]=\E_\Z\left[\sigma(\Z\w^*)\sigma(\Z\w^*)^\t\right] = \frac{1}{4}\left(\I + \1\1^\t\right)$. 
Furthermore, 
$$
\E_\Z\left[\sigma(\Z\w)\sigma(\Z\w^*)^\t\right]_{ii} = \E\left[1_{\{\Z_i^\t\w> 0, \Z_i^\t\w^*> 0\}}\right] = \frac{\pi-\theta(\w,\w^*)}{2\pi},
$$
and $\E_\Z\left[\sigma(\Z\w)\sigma(\Z\w^*)^\t\right]_{ij}=\frac{1}{4}$. 
So,
$$
\E_\Z\left[\sigma(\Z\w)\sigma(\Z\w^*)^\t\right] = \frac{1}{4}\left(\left(1-\frac{2\theta(\w,\w^*)}{\pi}\right)\I + \1\1^\t \right).
$$ 
We thus have proved (\ref{eq:loss}) by noticing that
\begin{align*}
f(\v,\w) = & \; \frac{1}{2}\big(\v^\t \E_\Z[\sigma(\Z\w)^\t\sigma(\Z\w)]\v - 2\v^\t\E_\Z[\sigma(\Z\w)^\t\sigma(\Z\w^*)]\v^* \\
& \; +(\v^*)^\t\E_\Z[\sigma(\Z\w^*)^\t\sigma(\Z\w^*)]\v^*\big).
\end{align*}
Next, since (\ref{eq:grad_v}) is trivial, we only show (\ref{eq:grad_w}). Since $\theta(\w,\w^*) = \arccos\left(\frac{\w^\t\w^*}{\|\w\|}\right)$ is differentiable w.r.t. $\w$ at $\theta(\w,\w^*)\in(0,\pi)$, we have
$$
\frac{\partial f}{\partial \w}(\v,\w) = \frac{\v^\t\v^*}{2\pi}\frac{\partial \theta}{\partial \w}(\w,\w^*) = -\frac{\v^\t\v^*}{2\pi}\frac{\|\w\|^2\w^* - (\w^\t\w^*)\w}{\|\w\|^3\sqrt{1-\frac{(\w^\t\w^*)^2}{\|\w\|^2}}} = -\frac{\v^\t\v^*}{2\pi\|\w\|}\frac{\Big(\I - \frac{\w\w^\t}{\|\w\|^2}\Big)\w^*}{\Big\| \Big(\I - \frac{\w\w^\t}{\|\w\|^2}\Big)\w^*\Big\|}.
$$
\end{proof}

\bigskip

\begin{proof}[{\bf Proof of Proposition {\ref{prop}}}]
Suppose $\v^\t \v^*=0$ and $\frac{\partial f}{\partial \v}(\v,\w) = \0$, then by Lemma \ref{lem:obj},
\begin{equation}\label{eq:3}
0 = \v^\t\v^* = (\v^*)^\t(\I + \1\1^\t)^{-1}\left( \left(1- \frac{2}{\pi}\theta(\w,\w^*)\right)\I + \1\1^\t \right)\v^*.
\end{equation}
From (\ref{eq:3}) it follows that
\begin{equation}\label{eq:5}
\frac{2}{\pi}\theta(\w,\w^*) (\v^*)^\t (\I + \1\1^\t)^{-1} \v^*= (\v^*)^\t(\I + \1\1^\t)^{-1}\left( \I + \1\1^\t \right)\v^* = \|\v^*\|^2.
\end{equation}
On the other hand, from (\ref{eq:3}) it also follows that
$$
\left(\frac{2}{\pi}\theta(\w,\w^*)-1\right) (\v^*)^\t (\I + \1\1^\t)^{-1} \v^* = (\v^*)^\t(\I + \1\1^\t)^{-1} \1(\1^\t \v^*) = \frac{(\1^\t\v^*)^2}{m+1},
$$
where $\I$ is an $m$-by-$m$ identity matrix, and we used $(\I + \1\1^\t) \1 = (m+1)\1$.
Taking the difference of the two equalities above gives
$$
(\v^*)^\t(\I + \1\1^\t)^{-1}\v^* =  \|\v^*\|^2 - \frac{(\1^\t\v^*)^2}{m+1}. 
$$
By (\ref{eq:5}), we have 
$\theta(\w,\w^*) = \frac{\pi}{2}\frac{(m+1)\|\v^*\|^2}{(m+1)\|\v^*\|^2 - (\1^\t\v^*)^2}$, which requires
$$
\frac{\pi}{2}\frac{(m+1)\|\v^*\|^2}{(m+1)\|\v^*\|^2 - (\1^\t\v^*)^2}<\pi, \;
\mbox{or equivalently, } \; (\1^\t\v^*)^2 < \frac{m+1}{2}\|\v^*\|^2.
$$
Otherwise, $\frac{\partial f}{\partial \v}(\v,\w)$ and $\frac{\partial f}{\partial \w}(\v,\w)$ do not vanish simultaneously, and there is no critical point.\\

\end{proof}

\bigskip

\begin{proof}[{\bf Proof of Lemma \ref{lem:lipschitz}}]
It is easy to check that $\|\I + \1\1^\t\| = m+1$. Invoking Lemma \ref{lem:angle}.1 gives
\begin{align*}
\left\|\frac{\partial f}{\partial \v}(\v,\w) - \frac{\partial f}{\partial \v}(\tv,\tw) \right\| = & \; \frac{1}{4}\left\|\big(\I + \1\1^\t \big)(\v-\tv) + \frac{2}{\pi}(\theta(\w,\w^*) - \theta(\tw,\w^*) )\v^* \right\| \\
\leq & \; \frac{1}{4}\left((m+1)\| \v - \tv \| + \frac{2\|\v^*\|}{\pi} |\theta(\w,\w^*) - \theta(\tw,\w^*)|\right) \\
\leq & \; \frac{1}{4}\left((m+1)\| \v - \tv \| + \frac{\|\v^*\|}{c} \left\|\w - \tw \right\|\right) \\
\leq & \; \frac{1}{4}\left( m+1 + \frac{\|\v^*\|}{c} \right) \|(\v, \w) - (\tv, \tw)\|.
\end{align*}
Using Lemma \ref{lem:angle}.2, we further have
\begin{align*}
\left\|\frac{\partial f}{\partial \w}(\v,\w) - \frac{\partial f}{\partial \w}(\tv,\tw) \right\| = & \; \left\|\frac{\v^\t\v^*}{2\pi\|\w\|}
\frac{\Big(\I - \frac{\w\w^\t}{\|\w\|^2}\Big)\w^*}{\Big\| \Big(\I - \frac{\w\w^\t}{\|\w\|^2}\Big)\w^*\Big\|}  -
\frac{\tv^\t\v^*}{2\pi\|\tw\|}
\frac{\Big(\I - \frac{\tw\tw^\t}{\|\tw\|^2}\Big)\w^*}{\Big\| \Big(\I - \frac{\tw\tw^\t}{\|\tw\|^2}\Big)\w^*\Big\|} \right\|\\
\leq & \; \left\|\frac{\v^\t\v^*}{2\pi\|\w\|}
\frac{\Big(\I - \frac{\w\w^\t}{\|\w\|^2}\Big)\w^*}{\Big\| \Big(\I - \frac{\w\w^\t}{\|\w\|^2}\Big)\w^*\Big\|}  -
\frac{\v^\t\v^*}{2\pi\|\tw\|}
\frac{\Big(\I - \frac{\tw\tw^\t}{\|\tw\|^2}\Big)\w^*}{\Big\| \Big(\I - \frac{\tw\tw^\t}{\|\tw\|^2}\Big)\w^*\Big\|} \right\| \\
& \; + \left\|\frac{\v^\t\v^*}{2\pi\|\tw\|}
\frac{\Big(\I - \frac{\tw\tw^\t}{\|\tw\|^2}\Big)\w^*}{\Big\| \Big(\I - \frac{\tw\tw^\t}{\|\tw\|^2}\Big)\w^*\Big\|}  -
\frac{\tv^\t\v^*}{2\pi\|\tw\|}
\frac{\Big(\I - \frac{\tw\tw^\t}{\|\tw\|^2}\Big)\w^*}{\Big\| \Big(\I - \frac{\tw\tw^\t}{\|\tw\|^2}\Big)\w^*\Big\|} \right\| \\
\leq & \; \frac{|\v^\t\v^*|}{2 \pi c^2 }\|\w-\tw \| + \frac{\|\v^*\|}{2\pi c}\|\v-\tv\|\\
\leq & \; \frac{(C+c)\|\v^*\|}{2\pi c^2}\|(\v, \w) - (\tv, \tw)\|.
\end{align*}
Combining the two inequalities above validates the claim.
\end{proof}

\bigskip

\begin{proof}[{\bf Proof of Lemma \ref{lem:cgd}}]
(\ref{eq:cgrad_v}) is true because
$\frac{\partial \ell}{\partial \v}(\v,\w;\Z)$ is linear in $\v$.
To show (\ref{eq:cgrad_w}), by (\ref{eq:cbp_w}) and the fact that $\mu^{\prime} = \sigma$, we have
\begin{align*}
\E_\Z \left[ \g(\v,\w;\Z)\right] = & \; \E_\Z \left[\left(\sum_{i=1}^m v_i \sigma(\Z^\t_i\w)  - \sum_{i=1}^m v^*_i\sigma(\Z^\t_i\w^*)  \right)\left(\sum_{i=1}^m \Z_i v_i \sigma(\Z^\t_i\w) \right) \right] \\
= & \; \E_\Z \left[\left(\sum_{i=1}^m v_i 1_{\{\Z^\t_i\w>0\}}  - \sum_{i=1}^m v^*_i1_{\{\Z^\t_i\w^*>0\}}  \right)\left(\sum_{i=1}^m 1_{\{\Z^\t_i\w>0\}} v_i\Z_i \right)\right].
\end{align*}
Invoking Lemma \ref{lem:basic}, we have
\begin{equation}
\E\left[ \Z_i 1_{\{\Z_i^\t \w>0, \Z_j^\t \w>0\}}\right] = 
\begin{cases}
\frac{1}{\sqrt{2\pi}} \frac{\w}{\|\w\|} & \mbox{if } i=j, \\
\frac{1}{2\sqrt{2\pi}} \frac{\w}{\|\w\|} & \mbox{if } i\neq j,
\end{cases}
\end{equation}
and
\begin{equation}
\E\left[ \Z_i 1_{\{\Z_i^\t \w>0, \Z_j^\t \w^* >0\}}\right] = 
\begin{cases}
\frac{\cos(\theta(\w,\w^*)/2)}{\sqrt{2\pi}} \frac{\frac{\w}{\|\w\|} + \w^* }{\left\|\frac{\w}{\|\w\|} + \w^*  \right\|} & \mbox{if } i=j,  \\
\frac{1}{2\sqrt{2\pi}} \frac{\w}{\|\w\|} & \mbox{if } i\neq j.
\end{cases}
\end{equation}
Therefore,
\begin{align*}
\E_\Z \left[ \g(\v,\w;\Z)\right] = & \;\sum_{i=1}^m v_i^2 \E\left[ \Z_i 1_{\{\Z_i^\t \w>0\}}\right] +  \sum_{i=1}^m \sum_{\overset{j=1}{j\neq i}}^m v_i v_j \E\left[ \Z_i 1_{\{\Z_i^\t \w>0, \Z_j^\t \w>0\}}\right]  \\
& \; - \sum_{i=1}^m v_i v_i^* \E\left[ \Z_i 1_{\{\Z_i^\t \w>0, \Z_i^\t \w^* >0\}}\right] - \sum_{i=1}^m \sum_{\overset{j=1}{j\neq i}}^m v_i v_j^* \E\left[ \Z_i 1_{\{\Z_i^\t \w>0, \Z_j^\t \w^*>0\}}\right] \\
= & \; \frac{1}{2\sqrt{2\pi}}\left(\|\v\|^2 + (\1^\t\v)^2 \right) \frac{\w}{\|\w\|} 
- \cos\left(\frac{\theta(\w,\w^*)}{2}\right)\frac{\v^\t\v^*}{\sqrt{2\pi}} \frac{\frac{\w}{\|\w\|} + \w^* }{\left\|\frac{\w}{\|\w\|} + \w^*  \right\|} \\
& \; - \frac{1}{2\sqrt{2\pi}}\left( (\1^\t\v)(\1^\t\v^*) - \v^\t\v^* \right)\frac{\w}{\|\w\|},
\end{align*}
which is exactly (\ref{eq:cgrad_w}).
\end{proof}

\bigskip

\begin{proof}[{\bf Proof of Lemma \ref{cor}}]
Notice that $\Big(\I - \frac{\w\w^\t}{\|\w\|^2}\Big)\w = \0$ and $\|\w^*\| = 1$, if $\theta(\w,\w_*)\neq 0, \pi$, then we have
\begin{align*}
 & \; \left\langle \E_\Z \Big[\g(\v,\w; \Z)\Big], \frac{\partial f}{\partial \w}(\v,\w) \right\rangle \\
= & \;
\cos\left(\frac{\theta(\w,\w^*)}{2}\right)\frac{(\v^{\t}\v^*)^2}{(\sqrt{2\pi})^3} \left\la \frac{1}{\|\w\|}
\frac{\Big(\I - \frac{\w\w^\t}{\|\w\|^2}\Big)\w^*}{\Big\| \Big(\I - \frac{\w\w^\t}{\|\w\|^2}\Big)\w^*\Big\|} , \frac{\w^*}{\left\|\frac{\w}{\|\w\|} + \w^*\right\|} \right\ra \\
= & \; \cos\left(\frac{\theta(\w,\w^*)}{2}\right)\frac{(\v^{\t}\v^*)^2}{(\sqrt{2\pi})^3} \frac{\|\w\|^2 - (\w^\t\w^*)^2}{\|\|\w\|^2\w^* - \w(\w^\t\w^*)\| \, \|\w+\|\w\|\w^*\|} \\
= & \; \cos\left(\frac{\theta(\w,\w^*)}{2}\right)\frac{(\v^{\t}\v^*)^2}{(\sqrt{2\pi})^3}  \frac{\|\w\|^2 - (\w^\t\w^*)^2}{\sqrt{\|\w\|^4 -\|\w\|^2(\w^\t\w^*)^2} \sqrt{2(\|\w\|^2+ \|\w\|(\w^\t \w^*))}} \\
= & \; \cos\left(\frac{\theta(\w,\w^*)}{2}\right)\frac{(\v^{\t}\v^*)^2}{4(\sqrt{\pi\|\w\|})^3}  \frac{\|\w\|^2 - (\w^\t\w^*)^2}{\sqrt{\|\w\|^2 -(\w^\t\w^*)^2} \sqrt{\|\w\|+ (\w^\t \w^*)}} \\
= & \; \cos\left(\frac{\theta(\w,\w^*)}{2}\right)
\frac{(\v^{\t}\v^*)^2\sqrt{1-\frac{\w^\t\w^*}{\|\w\|}}}{4(\sqrt{\pi})^3\|\w\|} \\
= & \; \cos\left(\frac{\theta(\w,\w^*)}{2}\right)\frac{(\v^{\t}\v^*)^2\sqrt{1 - \cos(\theta(\w,\w^*))}}{4(\sqrt{\pi})^3\|\w\|} \\
= & \; \frac{\sin\left(\theta(\w,\w^*)\right)}{2(\sqrt{2\pi})^3\|\w\|}(\v^{\t}\v^*)^2.
\end{align*}
\end{proof}

\bigskip

\begin{proof}[{\bf Proof of Lemma \ref{lem:correlated}}]
Denote $\theta:= \theta(\w,\w^*)$.
By Lemma \ref{lem:obj}, we have
\begin{align*}
\frac{\partial f}{\partial \v}(\v,\w) = \frac{1}{4}\big(\I + \1\1^\t \big) \v - \frac{1}{4}\left( \left(1-\frac{2\theta}{\pi} \right)\I + \1\1^\t \right)\v^*.
\end{align*}
Since $\|\w\|=1$, Lemma \ref{lem:cgd} gives
\begin{equation}\label{eq:g}
\E_\Z \Big[\g(\v,\w; \Z)\Big]  = \frac{h(\v,\v^*)}{2\sqrt{2\pi}}\w - \cos\left(\frac{\theta}{2}\right)\frac{\v^\t\v^*}{\sqrt{2\pi}}\frac{\w + \w^*}{\left\|\w + \w^*\right\|},
\end{equation}
where 
\begin{align}\label{eq:h}
h(\v,\v^*) = & \; \|\v\|^2+ (\1^\t\v)^2 - (\1^\t\v)(\1^\t\v^*) + \v^\t\v^* \notag \\
= & \; \v^\t\left( \I + \1\1^\t \right)\v - \v^\t(\1\1^\t - \I)\v^* \notag \\
= & \; \v^\t\left( \I + \1\1^\t \right)\v - \v^\t \left(\1\1^\t + \left(1-\frac{2\theta}{\pi}\right) \I \right)\v^* + 2\left(1 - \frac{\theta}{\pi}\right)\v^\t\v^*  \notag \\
= & \; 4 \v^\t \frac{\partial f}{\partial \v}(\v,\w) + 2\left(1 - \frac{\theta}{\pi}\right)\v^\t\v^*,
\end{align}
and by Lemma \ref{cor},
\begin{equation*}
\left\langle \E_\Z \Big[\g(\v,\w; \Z)\Big], \frac{\partial f}{\partial \w}(\v,\w) \right\rangle  = \frac{\sin\left(\theta\right)}{2(\sqrt{2\pi})^3}(\v^\t\v^*)^2.
\end{equation*}
Hence, for some $A$ depending only on $C$, we have
\begin{align*}
& \; \left\|\E_\Z \Big[\g(\v,\w; \Z)\Big] \right\|^2 \\
= & \; \left\| \frac{2 \v^\t \frac{\partial f}{\partial \v}(\v,\w)}{\sqrt{2\pi}} \w + \cos\left(\frac{\theta}{2}\right)\frac{\v^\t\v^*}{\sqrt{2\pi}}\left(\w - \frac{\w + \w^*}{\left\|\w + \w^*\right\|} \right) + \left(1-\frac{\theta}{\pi}-\cos\left(\frac{\theta}{2}\right)\right) \frac{\v^\t\v^*}{\sqrt{2\pi}}\w\right\|^2 \\
\leq & \; \frac{6C^2}{\pi} \left\|\frac{\partial f}{\partial \v}(\v,\w)\right\|^2 + \cos^2\left(\frac{\theta}{2}\right)\frac{3(\v^\t\v^*)^2}{2\pi}\left\|\w - \frac{\w + \w^*}{\left\|\w + \w^*\right\|} \right\|^2  \\
& \; + \left(1-\frac{\theta}{\pi}-\cos\left(\frac{\theta}{2}\right)\right)^2 \frac{3(\v^\t\v^*)^2}{2\pi} \\
\leq & \; \frac{6C^2}{\pi}\left\|\frac{\partial f}{\partial \v}(\v,\w)\right\|^2 + 
\cos^2\left(\frac{\theta}{2}\right) \frac{3\theta^2}{8\pi} (\v^\t\v^*)^2 
+  \left(1-\frac{\theta}{\pi}-\cos\left(\frac{\theta}{2}\right)\right)^2 \frac{3(\v^\t\v^*)^2}{2\pi} \\
\leq & \; \frac{6C^2}{\pi}\left\|\frac{\partial f}{\partial \v}(\v,\w)\right\|^2 + 
\frac{3\pi}{8}\cos^2\left(\frac{\theta}{2}\right) \sin^2\left(\frac{\theta}{2}\right) (\v^\t\v^*)^2 + \frac{3\sin(\theta)}{2\pi}(\v^\t\v^*)^2  \\
\leq & \; A\left(\left\|\frac{\partial f}{\partial \v}(\v,\w)\right\|^2 + \left\langle \E_\Z \Big[\g(\v,\w; \Z)\Big], \frac{\partial f}{\partial \w}(\v,\w) \right\rangle \right),
\end{align*}
 where the equality is due to (\ref{eq:g}) and (\ref{eq:h}), the first inequality is due to Cauchy-Schwarz inequality, the second inequality holds because the angle between $\w$ and $\frac{\w + \w^*}{\left\|\w + \w^*\right\|}$ is $\frac{\theta}{2}$ and 
$\left\|\w - \frac{\w + \w^*}{\left\|\w + \w^*\right\|} \right\| \leq \frac{\theta}{2}$, whereas the third inequality is due to $\sin(x)\geq \frac{2x}{\pi}$, $\cos(x)\geq 1-\frac{2x}{\pi}$, and 
$$
\left(1-\frac{2x}{\pi} -\cos(x)\right)^2\leq \left(\cos(x) - 1+ \frac{2x}{\pi} \right)\left(\cos(x) +  1 - \frac{2x}{\pi}\right) \leq \sin(x)(2\cos(x)) = \sin(2x),
$$
for all $x\in [0,\frac{\pi}{2}]$.
\end{proof}

\bigskip

\begin{proof}[{\bf Proof of Theorem \ref{thm}} ]
To leverage Lemma \ref{lem:lipschitz} and Lemma \ref{lem:correlated}, we would need the boundedness of $\{\v^t\}$. Due to the coerciveness of $f$ w.r.t $\v$, there exists $C_0>0$, such that $\|\v\|\leq C_0$ for any $\v\in\{\v\in\R^m: f(\v,\w)\leq f(\v^0,\w^0) \mbox{ for some } \w \}$. In particular, $\|\v^0\|\leq C_0$. Using induction, suppose we already have $f(\v^{t},\w^{t})\leq f(\v^0,\w^0)$ and $\|\v^t\|\leq C_0$. If $\w^t = \pm\w^*$, then $\w^{t+1} = \w^{t+2} = \dots = \pm\w^*$, and the original problem reduces to a quadratic program in terms of $\v$. So $\{\v^t\}$ will converge to $\v^*$ or $(\I + \1\1^\t)^{-1}(\1\1^\t - \I)\v^*$ by choosing a suitable step size $\eta$. In either case, we have $\left\|\E_\Z \Big[\frac{\partial \ell}{\partial \v}(\v^t,\w^t; \Z)\Big]\right\|$ and $\left\|\E_\Z \Big[\g(\v^t,\w^t; \Z)\Big]\right\|$ both converge to 0. Else if $\w^t \neq \pm\w^*$, we define for $a\in[0,1]$ that
$$
\v^t(a) :=  \v^t - a(\v^{t+1} - \v^t) = \v^t - a \eta \E_\Z  \left[\frac{\partial \ell}{\partial \v}(\v^t,\w^t;\Z)\right]
$$ 
and 
$$
\w^t(a) := \w^t - a(\w^{t+1/2} - \w^t) = \w^t - a\eta \E_\Z \left[\g(\v^t,\w^t; \Z)\right],
$$ 
which satisfy 
$$
\v^t(0) = \v^t, \; \v^t(1) = \v^{t+1}, \; \w^t(0) = \w^t, \; \w^t(1) = \w^{t+1/2}.
$$
Let us fix $0<c<1$ and $C\geq C_0$. By the expressions of $\E_\Z  \left[\frac{\partial \ell}{\partial \v}(\v^t,\w^t;\Z)\right]$ and $\E_\Z \left[\g(\v^t,\w^t; \Z)\right]$ given in Lemma \ref{lem:cgd}, and since $\|\w^t\|=1$, for sufficiently small $\tilde{\eta}$ depending on $C_0$, with $\eta\leq \tilde{\eta}$, it holds that $\|\v^t(a)\|\leq C$ and $\|\w^t(a)\|\geq c$ for all $a\in[0,1]$. Possibly at some point $a_0$ where $\theta(\w^t(a_0),\w^*) = 0$ or $\pi$, such that $\frac{\partial f}{\partial \w}(\v^t(a_0),\w^t(a_0))$ does not exist. Otherwise,  $\left\| \frac{\partial f}{\partial \w}(\v^t(a),\w^t(a)) \right\|$ is uniformly bounded for all $a\in[0,1]/\{a_0\}$, which makes it integrable over the interval $[0,1]$. Then we have
\begin{align}\label{eq:descent}
f(\v^{t+1}, \w^{t+1}) = & \; f(\v^{t+1}, \w^{t+1/2}) = f(\v^t+ (\v^{t+1} -\v^t), \w^t+ (\w^{t+1/2}-\w^t)) \notag \\
= & \; f(\v^t, \w^t) + \int_{0}^1 \left\la \frac{\partial f}{\partial \v}(\v^t(a),\w^t(a)) , \v^{t+1} -\v^t \right\ra \mathrm{d}a  \notag\\ 
& \; +  \int_{0}^1 \left\la \frac{\partial f}{\partial \w}(\v^t(a),\w^t(a)), \w^{t+1/2} - \w^t \right\ra   \mathrm{d}a \notag\\
= & \; f(\v^{t}, \w^{t}) + \left\la \frac{\partial f}{\partial \v}(\v^t,\w^t) , \v^{t+1} -\v^t  \right\ra +  \left\la \frac{\partial f}{\partial \w}(\v^t,\w^t) , \w^{t+1/2} -\w^t \right\ra \notag\\
& \; + \int_{0}^1 \left\la \frac{\partial f}{\partial \v}(\v^t(a),\w^t(a)) - \frac{\partial f}{\partial \v}(\v^t,\w^t) , \v^{t+1} -\v^t \right\ra \mathrm{d}a  \notag\\
& \; \qquad \qquad +  \int_{0}^1 \left\la \frac{\partial f}{\partial \w}(\v^t(a),\w^t(a)) - \frac{\partial f}{\partial \w}(\v^t,\w^t) , \w^{t+1/2} - \w^t \right\ra   \mathrm{d}a 
\notag\\
\leq & \;  f(\v^{t}, \w^{t}) -\left(\eta -\frac{L\eta^2}{2}\right)\left\|\frac{\partial f}{\partial \v}(\v^t,\w^t) \right\|^2 - \eta\left\la \frac{\partial f}{\partial \w}(\v^t,\w^t), \E_\Z \Big[\g(\v^t,\w^t; \Z)\Big] \right\ra \notag \\
& \; + \frac{L\eta^2}{2} \left\|\E_\Z \Big[\g(\v^t,\w^t; \Z)\Big] \right\|^2 \notag \\
\leq & \;  f(\v^{t}, \w^{t}) -\left(\eta -(1+A)\frac{L\eta^2}{2}\right) \left\|\frac{\partial f}{\partial \v}(\v^t,\w^t) \right\|^2 \notag \\
& \; - \left(\eta -\frac{AL\eta^2}{2}\right)  \left\la \frac{\partial f}{\partial \w}(\v^t,\w^t), \E_\Z \Big[\g(\v^t,\w^t; \Z)\Big] \right\ra. 
\end{align}
The third equality is due to the fundamental theorem of calculus. In the first inequality, we called Lemma \ref{lem:lipschitz} for $(\v^t, \w^t)$ and $(\v^t(a), \w^t(a))$ with $a\in[0,1]/\{a_0\}$. In the last inequality, we used Lemma \ref{lem:correlated}. 
So when $\eta < \eta_0:= \min\left\{\frac{2}{(1+A)L}, \tilde{\eta}\right\}$, we have $f(\v^{t+1},\w^{t+1})\leq f(\v^0,\w^0)$ and thus $\|\v^{t+1}\|\leq C_0$.

\medskip

Summing up the inequality (\ref{eq:descent}) over $t$ from $0$ to $\infty$ and using $f\geq 0$, we have
\begin{align*}
& \; \eta\sum_{t=0}^\infty \left(1 -(1+A)\frac{L\eta}{2}\right)\left\|\frac{\partial f}{\partial \v}(\v^t,\w^t) \right\|^2 + \left(1 -\frac{AL\eta}{2}\right)  \left\la \frac{\partial f}{\partial \w}(\v^t,\w^t), \E_\Z \Big[\g(\v^t,\w^t; \Z)\Big] \right\ra \\
\leq & \;  f(\v^0,\w^0)<\infty.
\end{align*}
Hence,
$$
\lim_{t\to \infty}\left\|\frac{\partial f}{\partial \v}(\v^t,\w^t)\right\| = 0
$$ 
and 
$$\lim_{t\to \infty} \left\la \frac{\partial f}{\partial \w}(\v^t,\w^t), \E_\Z \Big[\g(\v^t,\w^t; \Z)\Big] \right\ra  = 0.
$$
Invoking Lemma \ref{lem:correlated} again, we further have
$$
\lim_{t\to \infty}\left\|\E_\Z \Big[\g(\v^t,\w^t; \Z)\Big]\right\| = 0,
$$ 
which completes the proof.
\end{proof}











\end{document}